\documentclass{article}

\usepackage{times}
\usepackage{graphicx} 
\usepackage{subfigure} 
\usepackage{amsmath}
\usepackage{amsfonts}
\usepackage{amsthm}
\usepackage{bm}

\usepackage{natbib}

\usepackage{algorithm}
\usepackage{algorithmic}

\usepackage{hyperref}

\usepackage{xcolor}
\definecolor{darkgreen}{rgb}{0,0.6,0}

\title{A Correspondence Between Random Neural Networks and Statistical Field Theory}
\author{S. S. Schoenholz, J. Pennington, and J. Sohl-Dickstein\\
Google Brain\\
\textit{\{schsam, jpennin, jaschasd\}@google.com}
}

\newtheorem*{theorem}{Main Result}
\newtheorem{result}{Result}
\newtheorem{corollary}{Corollary}

\begin{document} 

\maketitle

\begin{abstract} 
A number of recent papers have provided evidence that practical design questions about neural networks may be tackled theoretically by studying the behavior of random networks. However, until now the tools available for analyzing random neural networks have been relatively \textit{ad hoc}. In this work, we show that the distribution of pre-activations in random neural networks can be exactly mapped onto lattice models in statistical physics. We argue that several previous investigations of stochastic networks actually studied a particular factorial approximation to the full lattice model. For random linear networks and random rectified linear networks we show that the corresponding lattice models in the wide network limit may be systematically approximated by a Gaussian distribution with covariance between the layers of the network. In each case, the approximate distribution can be diagonalized by Fourier transformation. We show that this approximation accurately describes the results of numerical simulations of wide random neural networks. Finally, we demonstrate that in each case the large scale behavior of the random networks can be approximated by an effective field theory.
\end{abstract} 

\section{Introduction}
\label{introduction}

Machine learning methods built on deep neural networks have had unparalleled success across a dizzying array of tasks ranging from image recognition~\citep{krizhevsky2012} to translation~\citep{wu2016} to speech recognition and synthesis~\citep{hinton2012}. Amidst the rapid progress of machine learning at large, a theoretical understanding of neural networks has proceeded more modestly. In part, this difficulty stems from the complexity of neural networks, which have come to be composed of millions or even billions~\cite{shazeer2017} of parameters with complicated topology. 

Recently, a number of promising theoretical results have made progress by considering neural networks that are, in some sense, random. For example,~\citet{choromanska2015} showed that random rectified linear neural networks could, with approximation, be mapped onto spin glasses;~\citet{saxe2014} explored the learning dynamics of randomly initialized networks;~\citet{daniely2016} and ~\citet{daniely2017} studied an induced duality between neural networks with random pre-activations and compositions of kernels;~\citet{raghu2016} and ~\citet{poole2016} studied the expressivity of deep random neural networks; and~\citet{schoenholz2016} studied information propagation through random networks. Work on random networks in the context of Bayesian neural networks has a longer history~\citep{neal1996priors,neal2012,cho2009}. Overall it seems increasingly likely that statements about randomly initialized neural networks might be able to inform practical design questions.

In a seemingly unrelated context, the past century has witnessed significant advances in theoretical physics, many of which may be attributed to the development of statistical, classical, and quantum \emph{field theories}. Field theory has been used to understand a remarkably diverse set of physical phenomena, ranging from the standard model of particle physics~\citep{weinberg1967}, which represents the sum of our collective knowledge about subatomic particles, to the codification of phase transitions using Landau theory and the renormalization group~\citep{chaikin2000}. Consequently, an extremely wide array of tools have been developed to understand and approximate field theories. 

In this paper we elucidate an explicit connection between random neural networks and statistical field theory. We demonstrate how well-established techniques in statistical physics can be used to study random neural networks in a quantitative and robust way. We begin by constructing an ensemble of random neural networks that we believe has a number of appealing properties. In particular, one limit of this ensemble is equivalent to studying neural networks after random initialization while another limit corresponds to probing the statistics of minima in the loss landscape. We then introduce a change of variables which shows that the weights and biases may be integrated out analytically to give a distribution over the pre-activations of the neural network alone. This distribution is identical to that of a statistical lattice model and this mapping is exact. We examine an expansion of our results as the network width grows large, and obtain concise and interpretable results for deep linear networks and deep rectified linear networks. We show that there exist well-defined mean field theories whose fluctuations may be fully characterized, and that there exist corresponding continuum field theories that govern the long wavelength behavior of these random networks. We compare our theory to simulations of random neural networks and find exceptional agreement. Thus we show that the behavior of wide random networks can be very precisely characterized.

This work leaves open a wide array of avenues that may be pursued in the future. In particular, the ensemble that we develop allows for a loss to be incorporated in the randomness. Looking forward, it seems plausible that statements about the distribution of local optima could be obtained using these methods. Moreover early training dynamics could be investigated by treating the small loss limit as a perturbation. Finally, generalization to arbitrary neural network architectures and correlated weight matrices is possible.

\section{Background} 

We now briefly discuss lattice models in statistical physics and their corresponding effective field theories, using the ubiquitous \emph{Ising model} as an example. Many materials are composed of a lattice of atoms. Magnets are such materials where the electrons orbiting the atoms all spin in the same direction. To model this behavior, physicists introduced a very simple model that involves ``spins'' placed on vertices of a lattice. In our simple example we consider spins sitting on a one-dimensional chain of length $L$ at sites indexed by $l$. The spins can be modeled in many ways, but in the simplest formulation of the problem we take $z^l\in\{-1,+1\}$. This represents spins that are either aligned or anti-aligned. For ease of analysis we consider a periodic chain defined so that the first site is connected to the last site and $z^{L+1} = z^0$.

The statistics of the spins in such a system are determined by the Boltzmann distribution which gives the probability of a configuration of spins to be given by $P(\{z^l\}) = e^{-\beta H(\{z^l\})}/Q$, where $\beta$ is the reciprocal of the thermodynamic temperature. Here,
\begin{equation}\label{eq:ising_model}
H(\{z^l\}) = -\frac J2\sum_l z^l z^{l+1}
\end{equation}
is the ``energy'' of the system, where $J$ is a coupling constant. This energy is minimized when all of the sites point in the same direction. The normalization constant $Q$ is the partition function, and is given by
\begin{equation}
Q = \sum_{z^0\in\{-1,1\}}\cdots\sum_{z^L\in\{-1,1\}}e^{-\beta H}.
\end{equation}
Despite the relative simplicity of the Ising model it has had enormous success in predicting the qualitative behavior of an extremely wide array of materials. In particular, it can successfully explain transitions of physical systems from disordered to ordered states.

In general, lattice models are often unwieldy, and their successful predictions are sometimes surprising given how approximately they treat interactions present in real materials. The resolution to both of these concerns lies in the realization that we often are most concerned with the behavior of systems at very large distances relative to the atomic separations. For example, in the case of magnets we care much more about the behavior of the whole material rather than how any two spins in the material behave. 

This led to the development of Effective Field Theory (EFT) where we compute a field $u(x)$ by averaging over many $z^l$ using, for example, a Gaussian window centered at $x$. The field $u(x)$ is defined at every point in space and so $x$ here corresponds to a continuous relaxation of $l$. It turns out that a number of features of the original model, such as symmetries and locality, survive the averaging process. In EFT we study a minimal energy that captures these essential (or long range) aspects of our original theory. 

The energy describing the EFT is typically a complicated function of $u$ and its derivatives, $\beta H[u(x), \nabla u(x), \nabla^2 u(x),\cdots]$. However, it has been shown that the long wavelength behavior of the system can successfully be described by considering a low-order expansion of $\beta H$. For the Ising model, for example, the effective energy is,
\begin{equation}\label{eq:eft_example}
\beta H = \frac12\int dx\left[mu^2(x) + vu^4(x) + K(\nabla u(x))^2\right].
\end{equation}
It can be shown, using the theory of irrelevant operators, that as long as $v>0$ higher order powers of $u$ do not change qualitative aspects of the resulting theory. Only even powers are allowed because the Ising model has global $z^l\to-z^l$ symmetry and the gradient term encodes the propensity of spins to align with one another. EFTs such as this one have been very successful at describing the large scale behavior of lattice models such as the Ising model.  A great triumph of modern condensed matter physics was the realization that phases of matter could be characterized by their symmetries in this way. 

A very large effort has been devoted to developing techniques to analyze lattice models and EFTs. Consequently, any theory that can be written as a lattice model or EFT has access to a wide array of approximate analytic and numerical techniques that can be leveraged to study it.  We will employ several of these techniques, such as the saddle point approximation, here. This paper is therefore a way of opening the door to use this extensive toolset to study neural networks.

\section{An Exact Correspondence}

Consider a fully-connected feed-forward neural network, $f:\mathbb R^{N_0}\to\mathbb R^{N_L}$, with $L$ layers of width $N_l$ parametrized by weights $W_{\alpha\beta}^l$, biases $b^l_\alpha$, and nonlinearities $\phi^l:\mathbb R\to\mathbb R$. The network is defined by the equation,
\begin{equation}
f(x) = \phi^{L+1}(W^L\phi^{L}(\cdots\phi^0(W^0 x +  b^0)\cdots) +  b^L).
\end{equation}
We equip this network with a loss function, $\ell(f(x), t)$, where $x\in\mathbb R^M$ is the input to the network and $t\in\mathbb R^N$ is a target that we would like to model. Given a dataset (inputs together with targets) given by $\{(x_i, t_i) : i\in\mathcal M\}$ we can define a ``data'' piece to our loss,
\begin{equation}
\mathcal L_D(f) = \sum_{i\in\mathcal M}\ell(f(x_i),  t_i).
\end{equation}
Throughout this text we will use Roman subscripts to specify an input to the network and Greek subscripts to denote individual neurons. We then combine this with an $L^2$ regularization term on both the weights and the biases to give a total loss,
\begin{equation}
\mathcal L(f) = \frac{J_D}2\mathcal L_D(f) +  \sum_{l=0}^L\left(\frac{N_l}{2\sigma_w^2}W^l_{\alpha\beta}W^l_{\alpha\beta} + \frac1{2\sigma_b^2}b^l_\alpha b^l_\alpha\right).
\end{equation}
Here we have introduced a parameter $J_D$ that controls the relative influence of the data on the loss. It can also be understood as the reciprocal of the regularization parameter. In this equation and in what follows, we adopt the Einstein summation convention in which there is an implied summation over repeated Greek indices.

To construct a stochastic ensemble of networks we must place a measure on the space of networks. As the objective we hope to minimize is the total loss, $\mathcal L$, with reference to \citet{jaynes1957} we select the maximum entropy distribution over $f$ subject to a measurement of the expected total loss, 
\begin{equation}
\langle \mathcal L\rangle = \int df P(f)\mathcal L(f).
\end{equation}
This gives a probability of finding a network $f$ to be given by, $P(f) = e^{-\mathcal L(f)}/Q$ where $Q$ is the partition function, 
\begin{equation}\label{eq:stochastic_partition_raw}
Q = \int[dW][db]e^{-\mathcal L}.
\end{equation}
Here we have introduced the notation $[dW] = \prod_l\prod_{\alpha\beta} dW_{\alpha\beta}^l$ and $[db] = \prod_l\prod_\alpha db^l_\alpha$ for simplicity.

While a choice of ensemble in this context will always be somewhat arbitrary, we argue that this particular ensemble has several interesting features that make it worthy of study. First, if we set $J_D=0$ then this amounts to studying the distribution of untrained, randomly initialized, neural networks with weights and biases distributed according to $W_{\alpha\beta}^l\sim\mathcal N(0,\sigma_w^2N_l^{-1})$ and $b^l_\alpha\sim\mathcal N(0,\sigma_b^2)$ respectively. This situation also amounts to considering a Bayesian neural network with a Gaussian distributed prior on the weights and biases as in~\citet{neal2012}. When $J_D$ is small, but nonzero, we may treat the loss as a perturbation about the random case. We speculate that the regime of small $J_D$ should be tractable given the work presented here. Studying the case of large $J_D$ will probably require methodology beyond what is introduced in this paper; however, if progress could be made in this regime, it would give insight into the distribution of minima in the loss landscape.

Our main result is to rewrite eq.~\eqref{eq:stochastic_partition_raw} in a form that is more amenable to analysis. In particular, the weights and biases may be integrated out analytically resulting in a distribution that depends only on the pre-activations in each layer. This formulation elucidates the statistical structure of the network and allows for systematic approximation. By a change of variables we arrive at the following theorem (for proof appendix \ref{app:main_result}).

\begin{theorem}
Through the change of variables, $z^l_{\alpha;i}  = W^l_{\alpha\beta}\phi^l(z^{l-1}_{\beta;i}) + b_\alpha^l$, the distribution over weights and biases defined by eq.~\eqref{eq:stochastic_partition_raw} can be converted to a distribution over the pre-activations of the neural network. When $N_l\gg |\mathcal M|$ the distribution over the pre-activations is described by a statistical lattice model defined by the partition function, 
\begin{align}
Q = \int[dz]\exp\left[-\frac{J_D}2\sum_{i\in\mathcal M}\ell(\phi^{L+1}(z^L_i),t_i) - \frac12\sum_{l=0}^L\left((\bm z^l_\alpha)^T(\bm \Sigma^l)^{-1}\bm z^l_\alpha + \ln|\bm\Sigma^l|\right)\right].\label{eq:stochastic_partition_full}
\end{align}
Here $(\bm z^l_\alpha)^T = (z^l_{\alpha;1},\cdots, z^l_{\alpha,|\mathcal M|})$ is a vector whose components are the pre-activations corresponding different inputs to the network and $\bm\Sigma^l_{ij} = \sigma_w^2N_l^{-1}\phi^l(z^{l-1}_{\alpha;i})\phi^l(z^{l-1}_{\alpha;j}) + \sigma_b^2$ is the correlation matrix between activations of the network from different inputs. The lattice is a one-dimensional chain indexed by layer $l$ and the ``spins'' are $z^l_i\in\mathbb R^{N_l}$. We term this class of lattice model the Stochastic Neural Network.
\end{theorem}

Eq.~\eqref{eq:stochastic_partition_full} is the full joint-distribution for the pre-activations in a random network with arbitrary activation functions and layer widths with no reference to the weights or biases. We see that this lattice model features coupling between adjacent layers of the network as well as between different inputs to the network. Finally, we see that the loss now only features the pre-activation of the last layer of the network. The input to the network and the loss therefore act as boundary conditions on the lattice.

There is a qualitative as well as methodological similarity between this formalism and the use of replica theory to study spin glasses~\cite{mezard1987}. Qualitatively, we notice that the replicated partition function in spin glasses involves the overlap function which measures the correlation between spins in different replicas while eq.~\eqref{eq:stochastic_partition_full} is naturally written in terms of $\Sigma^l$ which measures the correlation between activations due to different inputs to the network. Methodologically, when using the replica trick to analyze spin-glasses, one assumes that the interactions between spins are Gaussian distributed and shared between different replicas of the system; by integrating out the couplings analytically, the different replicas naturally become coupled. In this case the weights play a similar role to the interactions and their integration leads to a coupling between different the signals due to different inputs to the network.

Samples from a stochastic network with $\phi(z) = \tanh(z)$ can be seen in fig.~(\ref{fig:stochastic_network_samples}). 
\begin{figure}[!h]
\includegraphics[width=0.6\textwidth]{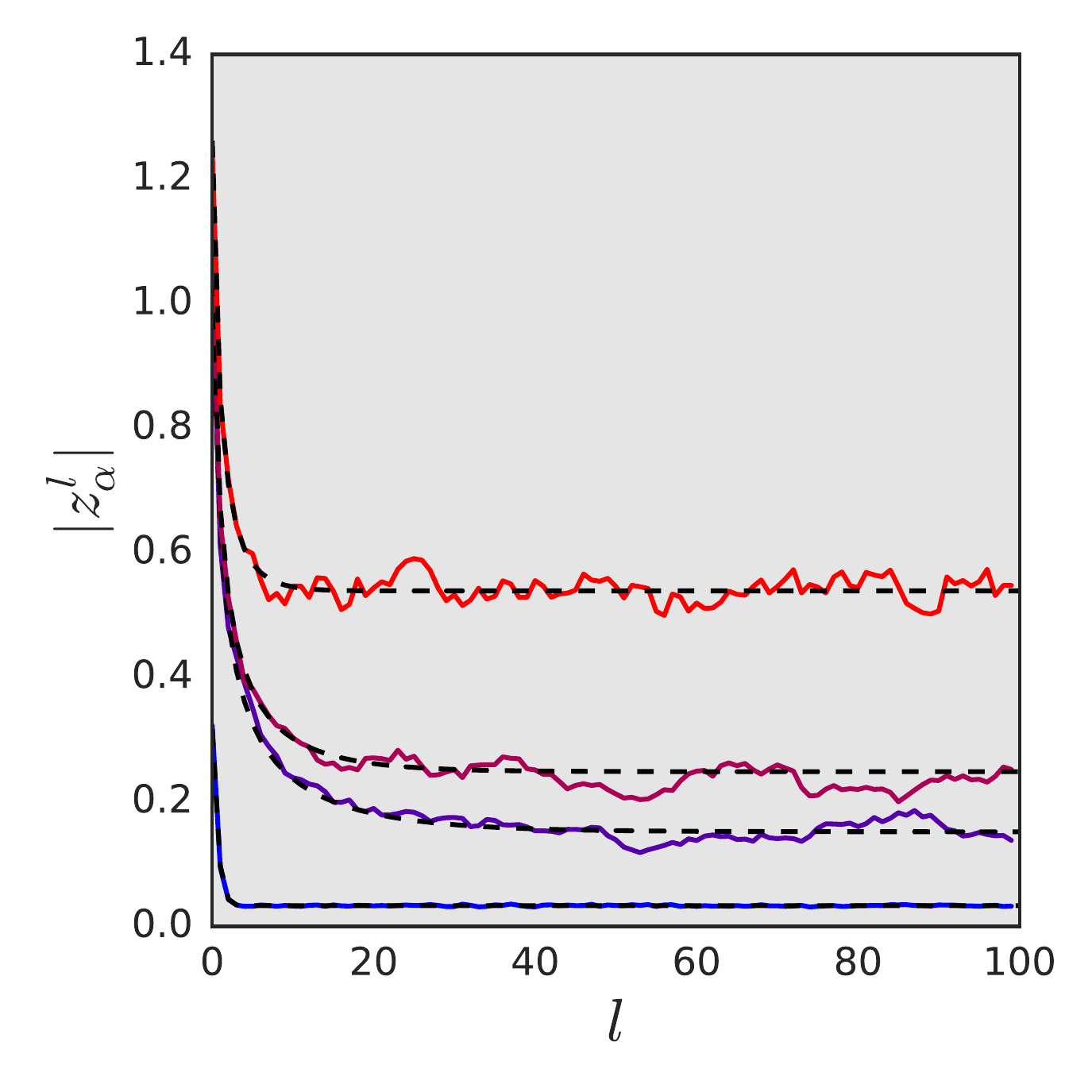}
\centering
\caption{Samples of the norm of pre-activations, $|z^l_\alpha|$, from an $L=100$ layer stochastic neural network with $\phi(z) = \tanh(z)$, $J_D = 0$, $N_l = 500$, and $\sigma_b^2 = 0.001$. The weight variance was changed from 0.1 (blue) to 1.5 (red). Dashed lines show the corresponding mean-field prediction.}
\label{fig:stochastic_network_samples}
\end{figure} 
In this framework we can see that the mean field approximation of \citet{poole2016} amounts to the replacement of $\phi^l(z_\alpha^l)\phi^l(z_\alpha^l)$ by $\langle\phi^l(z_\alpha^l)\phi^l(z_\alpha^l)\rangle$ where $\langle\rangle$ denotes an expectation. This procedure decouples adjacent layers and replaces the complex joint distribution over pre-activations by a factorial Gaussian distribution. As a result, this approximation is unable to capture any cross-layer fluctuations that might be present in random neural networks. We can see this in fig.~(\ref{fig:stochastic_network_samples}) where the black dashed lines denote the prediction of this particular mean field approximation. Note that while changes to the variance are correctly predicted, fluctuations are absent. Both~\citet{poole2016} and ~\citet{schoenholz2016} study this particular factorial approximation to the full joint distribution, eq.~\eqref{eq:stochastic_partition_full}. Additionally, the composition kernels of~\citet{daniely2016,daniely2017} can be viewed as studying correlation functions in this mean-field formalism over a broader class of network topologies.

The mean field theory of \citet{poole2016} is analytically tractable for arbitrary activation function and so it is interesting to study. However, the explicit independence assumption makes it an uncontrolled approximation, especially when generalizing to neural network topologies that are not fully connected feed-forward networks. Additionally, there are many interesting questions that one might wish to ask about correlations between pre-activations in different layers of random neural networks. Finally, it is unclear how to move beyond a mean field analysis in this framework. To overcome these issues, we pursue a more principled solution to eq.~\eqref{eq:stochastic_partition_full} by considering a controlled expansion for large $N_l$. 

To allow tractable progress, we limit the study in this paper to the case of a single input such that $|\mathcal M| = 1$. With this restriction, eq.~\eqref{eq:stochastic_partition_full} can be written explicitly as (see appendix \ref{app:main_result}),
\begin{align}
Q = \int[dz]\exp\Bigg[-\frac12\Bigg\{J_D\ell(\phi^{L+1}(z^L),t) &+ \frac{z_\alpha^0z_\alpha^0}{\sigma_w^2N_0^{-1}x_\beta x_\beta + \sigma_b^2} \nonumber\\
&+ \sum_{l=1}^L\frac{z_\alpha^lz_\alpha^l}{\sigma_w^2N_l^{-1}\phi^l(z^{l-1}_\beta)\phi^l(z^{l-1}_\beta) + \sigma_b^2}\nonumber\\
&+ \sum_{l=1}^LN_l\log(\sigma_w^2N_l^{-1}\phi^l(z^l_\alpha)\phi^l(z^l_\alpha) + \sigma_b^2) \Bigg\}\Bigg].\label{eq:stochastic_partition_full_single_input}
\end{align} 
While results involving the distribution of pre-activations resulting from a single input are an interesting first step we know from~\citet{poole2016, schoenholz2016, daniely2016} that correlations between the pre-activations due to different inputs is important when analyzing notions of expressivity and trainability. We therefore believe that extending these results to nontrivial datasets will be fruitful. To this end, it might be useful to take inspiration from the spin-glass community and seek to rephrase eq.~\eqref{eq:stochastic_partition_full} in terms of an overlap and to look for replica-symmetry breaking.

\section{The Stochastic Neural Network On A Ring}

With the stochastic neural network defined in eq.~\eqref{eq:stochastic_partition_full}, we consider a specific network topology that is unusual in machine learning but is commonplace in physics. In particular, as in the Ising model described above, we consider a stochastic network whose final layer feeds back into its first layer. Since this topology is incompatible with a loss defined in terms of network inputs and outputs, we set $J_D = 0$ in this case.
\begin{figure}[!h]
\includegraphics[width=0.6\textwidth]{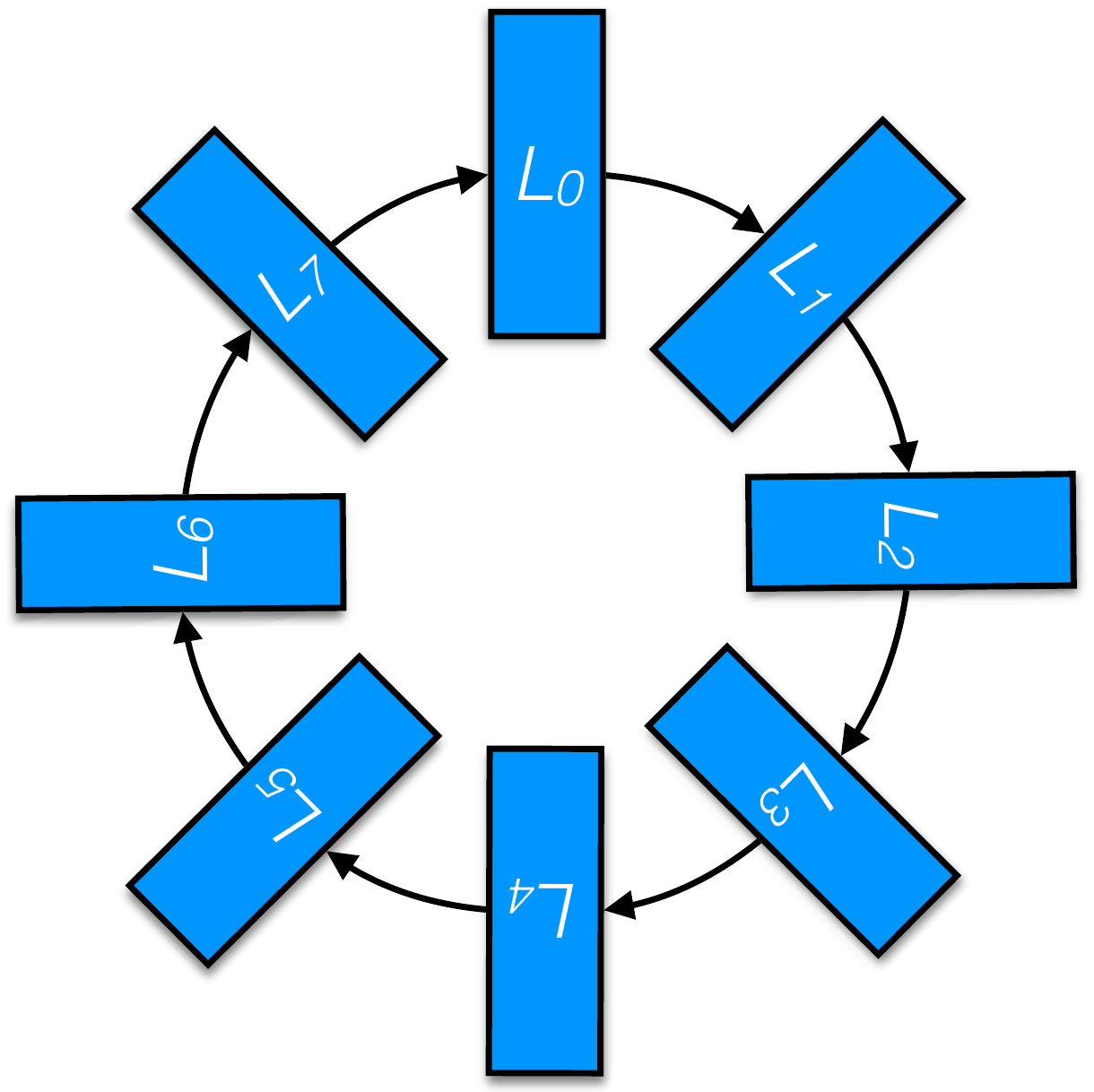}
\centering
\caption{A schematic showing the topology of the Stochastic Neural Network on a ring.}
\label{fig:stochastic_network_ring}
\end{figure} 
A schematic of this network can be seen in fig.~(\ref{fig:stochastic_network_ring}). The substantial advantage of considering this periodic topology is that we can neglect the effect of boundary conditions and focus on the ``bulk'' behavior of the network. The boundary effects can be taken into account once a theory for the bulk has been established. This method of dealing with lattice models is \emph{extremely} common. We additionally set $N_l = N$ and $\phi^l = \phi$ independent of layer. 

The stochastic network on a ring is described by the energy,
\begin{align}
 \mathcal L(\{z^l\}) = &\frac12\sum_{l=0}^L\Bigg\{\frac{z^l_\alpha z^l_\alpha}{\sigma_w^2N_l^{-1}\phi(z^{l-1}_\beta)\phi(z^{l-1}_\beta) + \sigma_b^2} + N\log(\sigma_w^2N^{-1}\phi(z^l_\beta)\phi(z^l_\beta) + \sigma_b^2)\Bigg\}\label{eq:stochastic_partition_ring}
\end{align}
subject to the identification $z^L_\alpha = z^{-1}_\alpha$. We will call this lattice model the stochastic neural network on a ring. For the remainder of this paper we will consider systematic approximations to eq.~\eqref{eq:stochastic_partition_ring}.

\section{Linear Stochastic Neural Networks}

To gain intuition for the stochastic network on a ring we will begin by considering a linear network with $\phi(z) = z$. In this case it is clear that the energy in eq.~\eqref{eq:stochastic_partition_ring} is isotropic. It is therefore possible to change variables into hyper-spherical coordinates and integrate out the angular part explicitly (which will give a constant factor that may be neglected). Consequently, the energy for the stochastic linear network is given by (see appendix~\ref{app:linear_network}),
\begin{align}
\mathcal L(\{r^l\}) = &\frac12\sum_{l=0}^L\Bigg\{\frac{(r^l)^2}{\sigma_w^2N^{-1}(r^{l-1})^2 + \sigma_b^2}- N\log\Bigg(\frac{(r^l)^2}{\sigma_w^2N^{-1}(r^l)^2+\sigma_b^2}\Bigg)\Bigg\}.\label{eq:stochastic_partition_linear}
\end{align}
where $(r^l)^2 = z^l_\alpha z^l_\alpha$.

A controlled approximation to eq.~\eqref{eq:stochastic_partition_linear} as $N\to\infty$ can be constructed using the Laplace approximation (sometimes called the saddle point approximation). The essence of the Laplace approximation is that integrals of the form $I = \int dx e^{-Af(x)}$ can be approximated by $I \approx e^{-Af(x^*)}\int dx e^{-A(x-x^*)^Tf''(x^*)(x-x^*)}$ as $A\to\infty$ where $x^*$ minimizes $f(x)$. Consequently, we first seek a minimum of eq.~\eqref{eq:stochastic_partition_linear} to expand around.  

We make the ansatz that there is a uniform configuration, $r^l = r^*$ independent of layer, that minimizes eq.~\eqref{eq:stochastic_partition_linear}. Under this assumption we find that for $\sigma_w^2 < 1$ there is an optimum when (see appendix~\ref{app:linear_network}),
\begin{equation}\label{eq:saddle_solution_linear}
r^* = \sqrt{\frac{N\sigma_b^2}{\Delta_w}}\,,
\end{equation}
where $\Delta_w = 1 - \sigma_w^2$ measures the distance to criticality. This solution can be tested by generating many instantiations of stochastic linear networks and then computing the average norm of the pre-activations after the transient from the input has decayed. 
\begin{figure}[!h]
\includegraphics[width=0.6\textwidth]{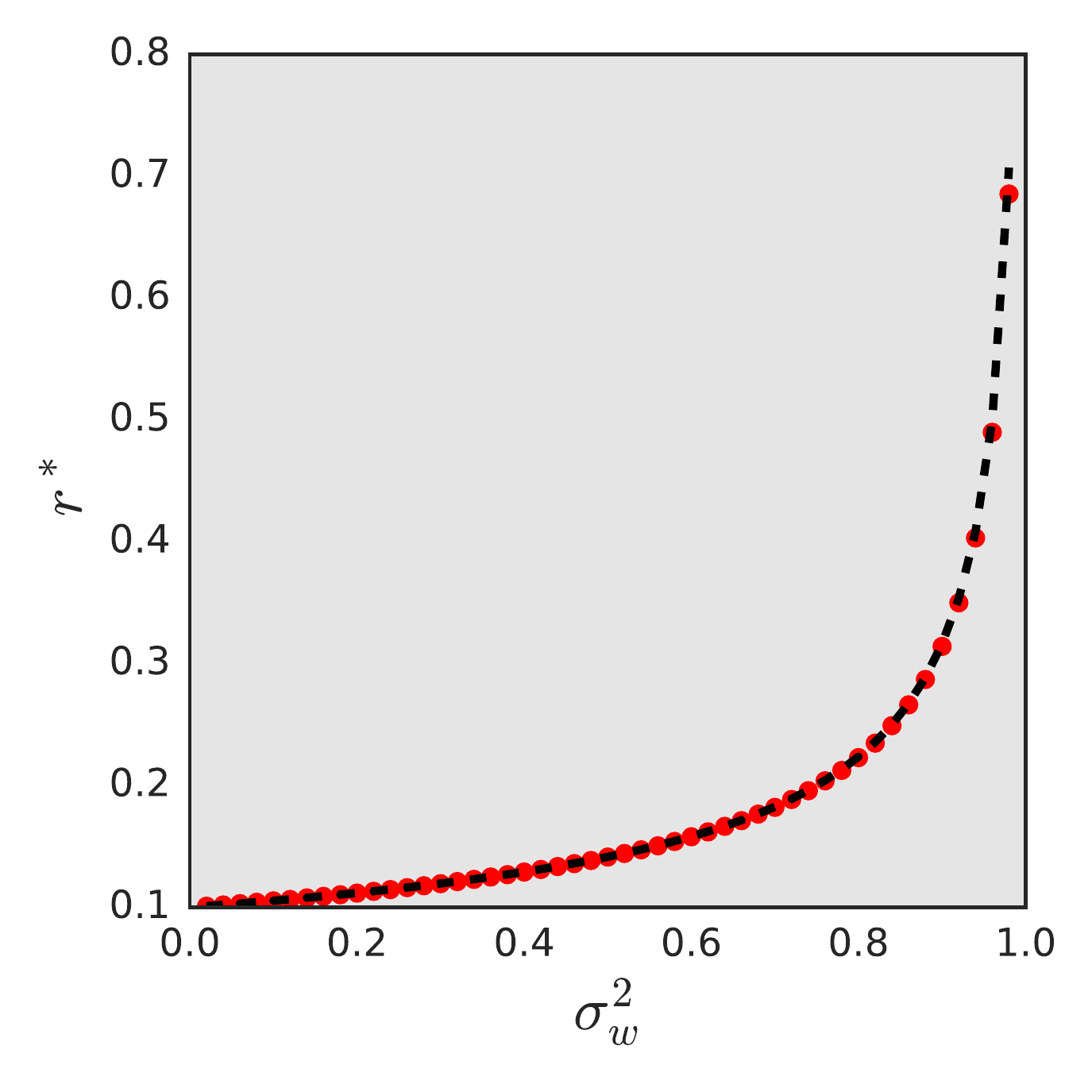}
\centering
\caption{The fixed point of the norm, $r^*$, for a stochastic linear network with $\sigma_b^2 = 0.01$, $L = 1024$, $N =200$. Measurements from instantiations of the network at different $\sigma_w^2$ are shown in red circles. The theoretical prediction is overlaid in black dashed lines.}
\label{fig:saddle_point_linear}
\end{figure} 
In fig.~(\ref{fig:saddle_point_linear}) we plot the empirical norm measured in this way against the theoretical prediction. We see excellent agreement between the numerical result and the theory\footnote{Note that while we are measuring the average norm of the linear stochastic network, we are predicting $r^*$ which is the mode of the distribution. However, these quantities are equal in the large $N$ limit of the Laplace approximation.}.

Nonuniform fluctuations around the minimum can now be computed. Let $r^l = r^* + \epsilon^l$ and expand the energy to quadratic order in $\epsilon^l$. Writing $U(\{\epsilon^l\}) = \mathcal L(\{r^*+\epsilon^l\}) - \mathcal L(\{r^*\})$ we find that (see appendix~\ref{app:linear_network}),
\begin{equation}
U(\{\epsilon^l\}) = \frac{\Delta_w}{\sigma_b^2}\sum_{l=0}^L\left[(1+\sigma_w^4)(\epsilon^l)^2 - 2\sigma_w^2\epsilon^l\epsilon^{l-1}\right].\label{eq:stochastic_partition_linear_saddle}
\end{equation}
As in the work of \citet{poole2016}, here we also approximate the behavior of the full joint distribution by a Gaussian. However, the Laplace approximation retains the coupling between layers and therefore is able to capture inter-layer fluctuations. 

Together eq.~\eqref{eq:saddle_solution_linear} and eq.~\eqref{eq:stochastic_partition_linear_saddle} fully characterize the behavior of the linear stochastic network as $N_l\to\infty.$ By expanding to beyond quadratic order, corrections of order $N_l^{-1}$ can be computed. One application of this would be to reprise the analysis of signal propagation in deep networks in~\citet{schoenholz2016}, but for networks of finite rather than infinite width.

As our network is topologically equivalent to a ring, we can perform a coordinate transformation of eq.~\eqref{eq:stochastic_partition_linear_saddle} to the Fourier basis by writing $\epsilon^l  = \sum_q \epsilon_q e^{-iql}$. To respect the periodic boundary conditions of the ring, $q$ will be summed from 0 to $2\pi$ in units of $2n\pi/L$. It follows that (see appendix~\ref{app:linear_network}),
\begin{equation}
U(\{\epsilon_q\}) = \frac{L\Delta_w}{\sigma_b^2}\sum_q\left\{(1+\sigma_w^4) - 2\sigma_w^2\cos q\right\}|\epsilon_q|^2.\label{eq:stochastic_partition_linear_saddle_fourier}
\end{equation} 
The Fourier transformation therefore diagonalizes eq.~\eqref{eq:stochastic_partition_linear_saddle} and so we predict that the different Fourier modes ought to be distributed as independent Gaussians. Since the variance of each mode is positive for $\sigma_w^2<1$, the optimum that we identified in eq.~\eqref{eq:saddle_solution_linear} is indeed a minimum.  

This calculation gives very precise predictions about the behavior of pre-activations in wide, deep stochastic networks.
\begin{figure}[!h]
\includegraphics[width=0.6\textwidth]{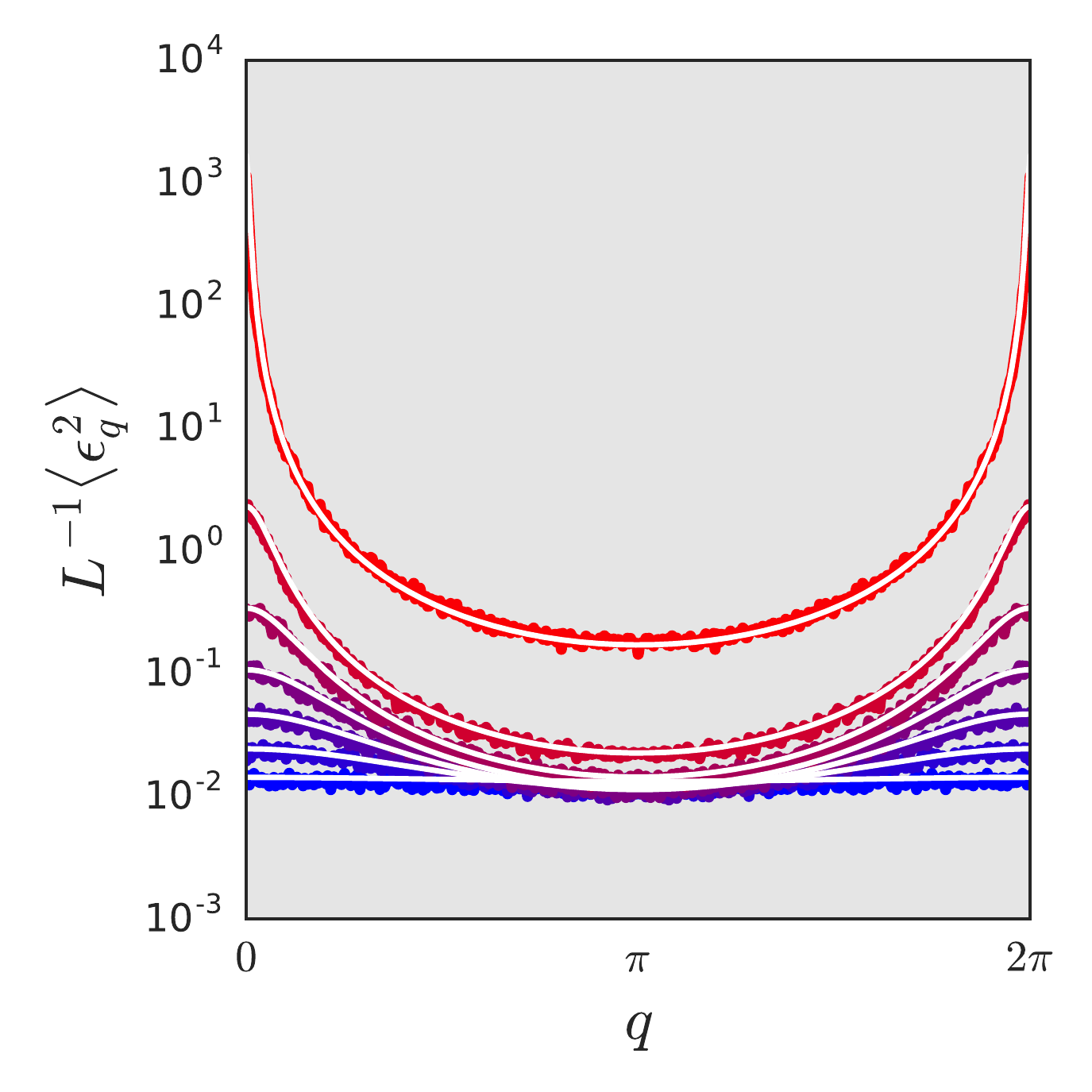}
\centering
\caption{The statistics of the Fourier transform of fluctuations in deep linear stochastic neural networks. This figure offers a comparison between the fluctuations sampled from stochastic neural networks (colored lines) and our theoretical predictions (white lines). The networks are of depth $L=1024$, width $N=200$, and $\sigma_b^2=0.01$. The colors denote different values of $\sigma_w^2$ in the set $0.02$ (blue), $0.18$, $0.34$, $0.5$, $0.66$, $0.82$, $0.98$ (red).}
\label{fig:saddle_point_fluctuations_linear}
\end{figure} 
To test these predictions we generate $M = 200$ samples from linear stochastic networks of width $N = 200$ and depth $L = 1024$. For each sample we take the norm of the pre-activations in the last $512$ layers of the network and compute the fluctuation of the pre-activation around $r^*$ (eq.~\eqref{eq:saddle_solution_linear}). For each sample we then compute the FFT of the norm of the pre-activations. Finally, we compute the variance of each Fourier mode (for more details and plots see appendix~\ref{app:linear_network_numerics}). We plot the results of this calculation in fig.~(\ref{fig:saddle_point_fluctuations_linear}) for different values of $\sigma_w^2$. In each case we see strong agreement between our numerical experiments and the prediction of our theory. Note that the factorial Gaussian approximation discussed briefly above is unable to capture these fluctuations.

The long wavelength behavior of fluctuations in the deep linear network is well described by an effective field theory. This effective field theory can be constructed by expanding eq.~\eqref{eq:stochastic_partition_linear_saddle_fourier} to quadratic order in $q$, approximating sums by integrals and differences by derivatives. We find that the effective field theory is defined by the energy (see appendix~\ref{app:linear_network}),
\begin{equation}
U[\epsilon(x)] = \frac{\Delta_w}{\sigma_b^2}\int dx\left[\Delta_w^2(\epsilon(x))^2 + \sigma_w^2\left(\frac{\partial \epsilon(x)}{\partial x}\right)^2\right].
\end{equation}
We note that this field theory features explicitly $\epsilon(x)\to-\epsilon(x)$ as well as $x\to -x$ symmetry. Perhaps expectedly this implies that information can equally travel forward and backwards through the network. 

Both the effective field theory and the lattice model have long wavelength fluctuations that are given by the $q\to 0$ limit of eq.~\eqref{eq:stochastic_partition_linear_saddle_fourier},
\begin{equation}
U(\{\epsilon_q\}) \approx \frac{L\Delta_w\sigma_w^2}{\sigma_b^2}\sum_q\left\{\frac{\Delta_w^2}{\sigma_w^2} + q^2\right\}|\epsilon_q|^2.
\end{equation}
Given this equation we can read off the length-scale governing fluctuations to be $\xi = \sigma_w/\Delta_w.$ We therefore see that stochastic linear networks feature a phase transition at $\Delta_w = 0$ with an accompanying diverging depth-scale in the fluctuations. 

\section{Rectified Linear Stochastic Neural Networks}

Having discussed the linear stochastic neural network we now move on to the more complicated case of the stochastic neural network on a ring with rectified linear activations, $\phi(z) = \max(0, z)$. Again we seek to construct the Laplace approximation to eq.~\eqref{eq:stochastic_partition_ring}.

In this case we notice that the norm squared of any $z^l$ decomposes into two terms, $(z^l)^2 = (z^l_+)^2 + (z^l_-)^2$. Here, $z^l_+$ and $z^l_-$ are the vectors of positive and negative components of $z^l$ respectively. With this decomposition, the energy for the rectified linear stochastic neural network can be written as,
\begin{align}
 \mathcal L(\{z^l\}) = &\frac12\sum_{l=0}^L\Bigg\{\frac{(z^l_+)^2 + (z^l_-)^2}{\sigma_w^2N^{-1}(z^l_+)^2 + \sigma_b^2} + N\log(\sigma_w^2N^{-1}(z^l_+)^2 + \sigma_b^2)\Bigg\}\label{eq:stochastic_energy_rectified_linear}.
\end{align}
The integral over each $z^l$ can be decomposed as a sum of integrals over each of the $2^N$ different orthants. In each orthant, the set of positive and negative components of $z^l$ is fixed; Consequently, we may apply independent hyperspherical coordinate transformations to $z^l_+$ and to $z^l_-$ within each orthant.

With this in mind, let $k_l$ be the number of positive components of $z^l$ in a given orthant with the remaining $N - k_l$ components being negative. It is clear that the number of orthants with $k_l$ positive components will be ${N\choose k_l}$. The partition function for the rectified linear network can therefore be written as (see appendix~\ref{app:rectified_linear_network}),
\begin{align}
Q &= 2\left(\frac{\sqrt\pi}{2}\right)^N\prod_l \sum_{k_l=0}^N{N\choose k_l}\frac1{\Gamma\left(\frac{N-k_l}2\right)\Gamma\left(\frac{k_l}2\right)}\int dr^l_+dr^l_-(r^l_+)^{N-k_l-1}(r^l_-)^{k_l-1}\nonumber\\&\hspace{3pc}\times\exp\left[-\frac12\sum_{l=0}^L\left(\frac{(r_+^l)^2 + (r^l_-)^2}{\sigma_w^2N^{-1}(r^{l-1}_+)^2 + \sigma_b^2} + N\log(\sigma_w^2N^{-1}(r^l_+)^2+\sigma_b^2)\right)\right].\label{eq:stochastic_partition_rectified_linear}
\end{align}
Here $r^l_+$ and $r^l_-$ is the norm of the positive and negative components of the pre-activations respectively. In the $N\to\infty$ limit, the sum over orthants can be converted into an integral and the $\Gamma$ functions can be approximated using Stirling's formula. We therefore see that, unlike in the case of linear networks, the lattice model for rectified linear networks contains three interaction fields, $r^l_+$, $r^l_-$, and $k_l$.

As in the linear case, we can now construct the Laplace approximation for this network. We first make an ansatz of a constant solution, $r^l_{+/-} = r^*_{+/-}$ and $k^l = k^*$, independent of the layer $l$. Solving for the minimum of the energy we arrive at the following saddle point conditions (see appendix~\ref{app:rectified_linear_network}),
\begin{align}
&r^*_{+/-} = \sqrt{\frac{N\sigma_b^2}{2(1-\sigma_w^2/2)}}\,, & k^* = \frac N2.
\end{align} 
Perhaps this result should not be surprising given the symmetry of the random weights. We expect that in the $N\to\infty$ limit the network will settle into a state where half the pre-activations are negative and half the pre-activations are positive.
\begin{figure}[!h]
\includegraphics[width=0.6\textwidth]{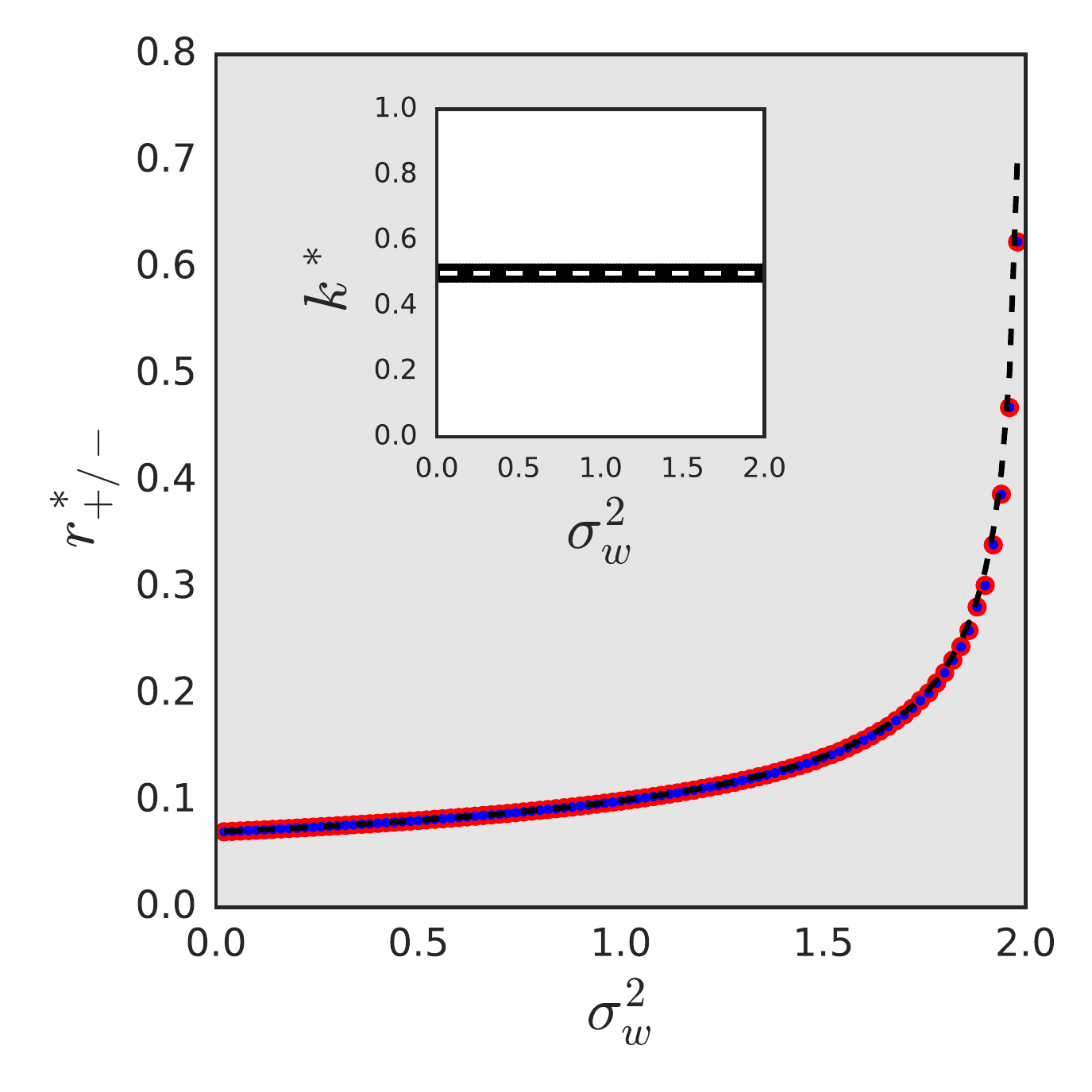}
\centering
\caption{The fixed point of the positive (red) and negative (blue) components of the norm, $r^*_{+/-}$, for a stochastic rectified linear network with $\sigma_b^2 = 0.01$, $L = 1024$, $N =200$. Measurements from instantiations of the network at different $\sigma_w^2$ are shown. The theoretical prediction is overlaid in black dashed lines. The inset shows that measured values for $k^*$ (black) compared with the theoretical prediction (dashed white).}
\label{fig:stochastic_partition_relu_saddle}
\end{figure} 
We can test the results of this prediction in fig.~(\ref{fig:stochastic_partition_relu_saddle}) by sampling $M = 200$ instances of 1024 layer deep rectified linear stochastic neural networks with $\sigma_w^2\in(0,2)$. As in the case of the deep linear stochastic network we see excellent agreement between theory and numerical simulation.

Nonuniform fluctuations around the saddle point can once again be computed. To do this we write $r^l_{+/-} = r^*_{+/-} + \epsilon^l_{+/-}$ and $k^l = k^* + \epsilon_k^l$. We now expand the energy and make the substitutions $\tilde\epsilon^l_{+/-} = \sqrt{2(1-\sigma_w^2/2)/\sigma_b^2}\epsilon^l_{+/-}$ and $\tilde \epsilon^l_k=\epsilon^l_k/\sqrt{N}$ to find an energy cost for fluctuations (see appendix~\ref{app:rectified_linear_network}),
\begin{align}
U = \frac12\sum_{l=0}^L\Bigg(&(1+\sigma_w^4/2)(\tilde\epsilon^l_+)^2 + (\tilde\epsilon^l_-)^2 + 3(\tilde\epsilon^l_k)^2+\tilde\epsilon^l_k(\tilde\epsilon^l_+ - \tilde\epsilon^l_-) - \sigma_w^2\tilde\epsilon_+^{l-1}(\tilde\epsilon_+^l + \tilde\epsilon_-^l)\Bigg).
\end{align}
We can understand some of these fluctuations in an intuitive way, for example fluctuations in the norm of the fraction of positive components and the norm of the negative components are anti-correlated. But in general rectified linear networks have subtle and interesting fluctuations, and to our knowledge this work presents the first quantitative theoretical description of the statistics of random rectified linear networks. We note in passing that that the fully factorial mean field theory would not be able to capture any of the anisotropy in the fluctuations identified here. 

As in the linear case, the layer-layer coupling can be diagonalized by moving into Fourier space. In the rectified linear case, however, this transformation retains covariance between the different fluctuations. In particular, we can write the energy in Fourier space as $U = \frac12\sum_q \bm\epsilon^\dag(q) \bm\Sigma^{-1}(q) \bm \epsilon(q)$ where $\bm\epsilon^\dag(q) = \begin{pmatrix}\epsilon_+^{-q} & \epsilon_-^{-q} & \epsilon_k^{-q}\end{pmatrix}$ is a vector of fluctuations and
\begin{align}
\bm \Sigma^{-1}(q) &= \begin{pmatrix}
1+\sigma_w^4/2 -\sigma_w^2\cos q & -\frac12\sigma_w^2e^{-iq} & \frac12 \\
-\frac12\sigma_w^2e^{iq} & 1 & -\frac12 \\
\frac12 & -\frac12 & 3
\end{pmatrix}
\end{align}
is the Fourier space inverse covariance matrix between different fields (see appendix~\ref{app:rectified_linear_network}).

We can compare our theoretical predictions for the covariance matrix against numerical results generated in an analogous manner to the linear case.
\begin{figure}[!h]
\includegraphics[width=\textwidth]{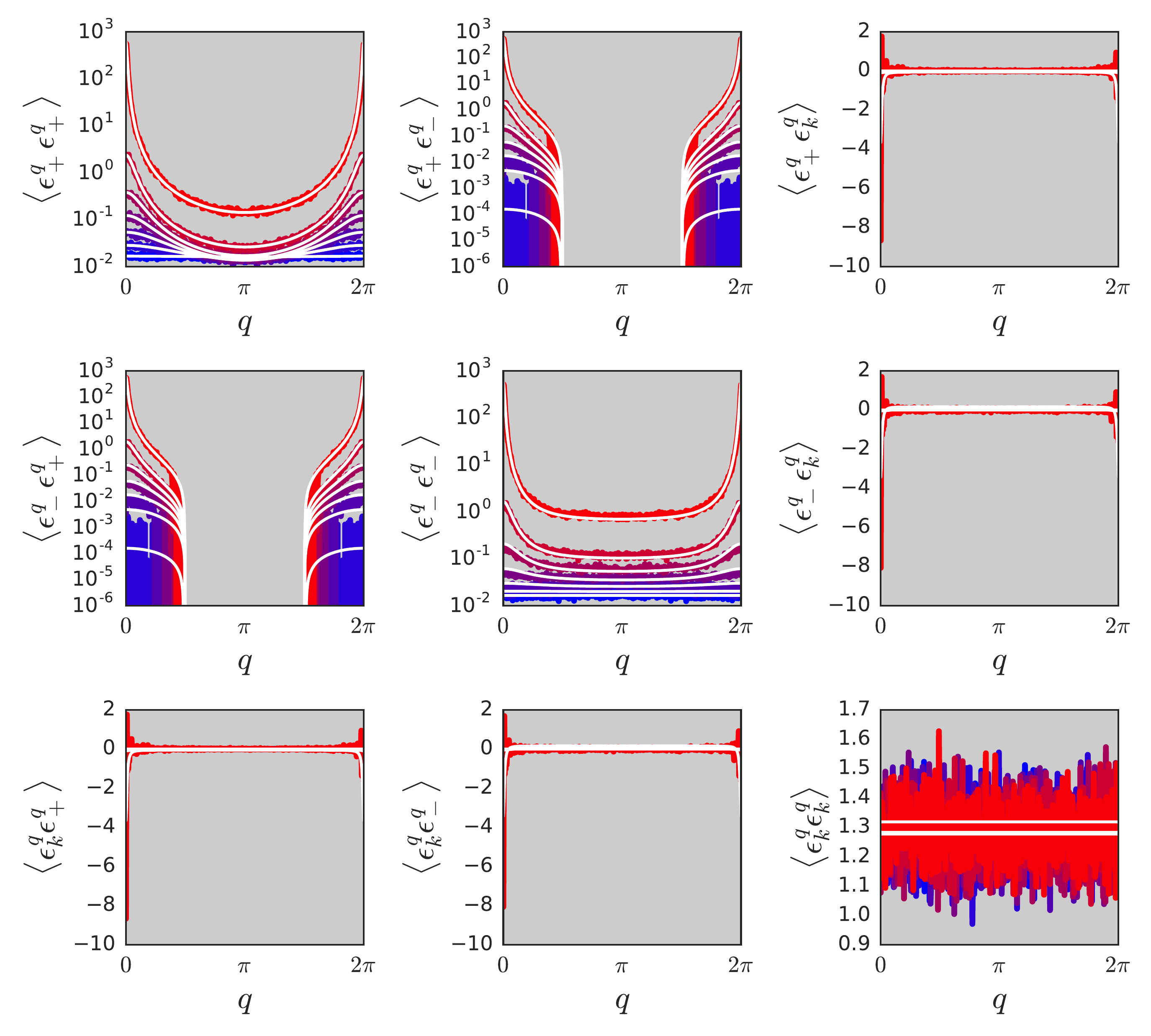}
\centering
\caption{The statistics of the Fourier transform of fluctuations in deep rectified stochastic neural networks. This figure offers a comparison between the fluctuations sampled from stochastic neural networks (colored lines) and our theoretical predictions (white lines). The networks are of depth $L=1024$, width $N=200$, and $\sigma_b^2=0.01$. Different colored curves denote different values of $\sigma_w^2$. In particular we show $\sigma_w^2=0.02$ (blue), $0.34$, $0.66$, $0.98$, $1.3$, $1.62$, $1.94$ (red). Different components of the covariance matrix are shown in different subplots.}
\label{fig:saddle_point_fluctuations_relu}
\end{figure}
The results of this comparison can be seen in fig.~(\ref{fig:saddle_point_fluctuations_relu}) for different elements of the covariance matrix and different values of $\sigma_w^2\in(0,2)$. As in the linear case we see excellent agreement between the theoretical predictions and the numerical simulations. Finally, we can complete our analysis by computing an effective field theory that governs long wavelength fluctuations (see appendix~\ref{app:rectified_linear_network}).

Once again we can identify an effective field theory that governs long wavelength fluctuations. We find that it is given by (see appendix~\ref{app:rectified_linear_network}),
\begin{align}
U = \frac12\int dx\Bigg[&(1 - \sigma_w^2 +\sigma_w^4/2)(\epsilon_+(x))^2 + (\epsilon_-(x))^2 + 3(\epsilon_k(x))^2 + \sigma_w^2\epsilon_+(x)\epsilon_-(x) \nonumber\\
& + \epsilon_k(x)(\epsilon_+(x) - \epsilon_-(x)) + \sigma_w^2\left(\frac{\partial\epsilon_+(x)}{\partial x}\right)^2  + \sigma_w^2\frac{\partial \epsilon_+(x)}{\partial x}\epsilon_-(x)  \Bigg].
\end{align}
Note that unlike in the case of the stochastic linear network both the $\epsilon\to-\epsilon$ and $x\to-x$ symmetries are broken when acting on any given field. This symmetry breaking makes sense since the network treats the different fields quite asymmetrically and the forward and backward propagation dynamics are quite different. In physics, the symmetries and symmetry breaking have been shown to dictate the behavior of systems over large regions of their parameters. Thus, as in Landau theory, many systems are classified based on the symmetries they possess. The presence of this symmetry breaking between linear networks and rectified-linear networks suggests that such an approach might be fruitfully applied to neural networks. As with the deep linear network, the long-wavelength limit of the effective field theory and the lattice model agree.

\section{Discussion}

Here we have shown that for fully-connected feed forward neural networks there is a correspondence between random neural networks and lattice models in statistical physics. While we have not discussed it here, this correspondence actually holds for a very large set of neural network topologies. Lattice models can also be constructed for ensembles of random neural networks that have weights and biases whose distributions are more complicated than factorial Gaussian. In general, the effect of nontrivial network topology and correlations between weights will be to couple spins in the lattice model. Thus, the topology of the neural network will generically induce a topology of the corresponding lattice model. For example, convolutional networks will have corresponding lattice models that feature interactions between the set of all the pre-activations in a given layer that share a filter.

As in physics, it seems likely that lattice models for complex neural networks will be fairly intractable compared to the relatively simple examples presented here. On the other hand, the success of effective field theories at describing the long wavelength fluctuations of random neural networks suggests that even complex networks may be tractable in this limit. Moreover, as neural networks get larger and more complex the behavior of long wavelength fluctuations will become increasingly relevant when thinking about the behavior of the neural network as a whole. 

We believe it is likely that there exist universality classes of neural networks whose effective field theories contain the same set of relevant operators. Classifying neural networks in this way would allow us to make statements about the behavior of entire classes of networks. This would transition the paradigm of neural network design away from specific architectural decisions towards a more general discussion about which class of models was most suitable for a specific problem. 

Finally, we note that there is has been significant effort made to understand biological neural activity leveraging similar analogies to lattice models and statistical field theory. Notably, \citet{elad2006} noticed that Ising-like models can quantitatively capture the statistics of neural activity in vertebrate retina; \citet{buice2013} developed field theoretic extensions to older mean-field theories of populations of neurons; far earlier, \citet{ermentrout1979} used similar techniques to investigate how hallucinations between two similar patterns might come about. By placing artificial neural networks into the context of field-theory it may be possible to find subtle relationships with their biological counterparts.



\section{Appendix}

\subsection{Proof of the Main Result}\label{app:main_result}

\noindent In this section we prove the main result of the paper. We do so in two steps. First we examine the partition function,
\begin{align}\label{seq:stochastic_partition_raw}
Q = \int&[dW][db]\exp\Bigg[-\frac{J_D}2\sum_{i\in\mathcal M} \ell(f(x_i), t_i)  - \frac12\sum_{l=0}^L\Bigg(\frac{N_l}{\sigma_w^2}W^l_{\alpha\beta}W^l_{\alpha\beta} + \frac1{\sigma_b^2}b^l_\alpha b^l_\alpha\Bigg)\Bigg]
\end{align}
and introduce the pre-activations at the cost of adding $\delta$-function constraints. We use the Fourier representation of these constraints to bring them into the exponent. This requires introducing auxiliary variables that enforce the constraints. Once in this form it becomes apparent that the weights and biases are Gaussian distributed and may therefore be integrated out explicitly. Finally we integrate out the constraints that we introduced in the preceding step to convert the distribution into a distribution over the pre-activations alone.

\begin{result}\label{sec:partition_function_fict}
The partition function for the maximum entropy distribution of a fully-connected feed-forward neural network can be written as,
\begin{align}\nonumber
Q = &\int[d\Omega]\exp\Bigg[-\sum_{i\in\mathcal M}\Bigg\{\frac{J_D}2\ell(\phi^{L+1}(z^L_i),t_i) + i\lambda_{\alpha;i}^0(z^0_{\alpha;i} - W^0_{\alpha\beta}x_{\beta;i} - b^0_{\alpha;i})  \\ & + \sum_{l=1}^Li\lambda^l_{\alpha;i}(z^l_{\alpha;i}-W^l_{\alpha\beta}\phi^l(z_{\beta;i}^{l-1}) - b^l_\alpha)\Bigg\} -\sum_{l=0}^L\Bigg(\frac{N_l}{2\sigma_w^2}W^l_{\alpha\beta}W^l_{\alpha\beta} + \frac1{2\sigma_b^2}b^l_\alpha b^l_\alpha\Bigg)\Bigg]\label{seq:partition_function_fict}
\end{align}
where we have let $[d\Omega] = [dW][db][dz][d\lambda]$ for notational convenience. 
\end{result}
\begin{proof}
To demonstrate this result we begin with eq.~\eqref{seq:stochastic_partition_raw} and iteratively use $\delta$-functions to change variables to the pre-activations. Explicitly writing the neural network out in eq.~\eqref{seq:stochastic_partition_raw} gives,
\begin{align}
Q&=\int[dW][db]\exp\Bigg[-\frac{J_D}2\sum_{i\in\mathcal M}\ell(\phi^L(\cdots \phi^1(W^0x_i+b^0)\cdots ), t_i) \nonumber\\&\hspace{7pc} - \sum_{l=0}^L\left(\frac{N_l}{2\sigma_w^2}W^l_{\alpha\beta}W^l_{\alpha\beta} + \frac1{2\sigma_b^2}b^l_\alpha b^l_\alpha\right)\Bigg]\\
&=\int[dW][db]\left(\prod_{i\in\mathcal M}\exp\left[-\frac{J_D}2\sum_i\ell(\phi^L(\cdots \phi^1(W^0x_i+b^0)\cdots ), t_i)\right]\right)\nonumber\\&\hspace{6pc}\times\exp\left[ - \sum_{l=0}^L\left(\frac{N_l}{2\sigma_w^2}W^l_{\alpha\beta}W^l_{\alpha\beta} + \frac1{2\sigma_b^2}b^l_\alpha b^l_\alpha\right)\right]\\
&=\int[dW][db]\Bigg(\prod_{i\in\mathcal M}\int[dz^0_i]\exp\left[-\frac{J_D}2\sum_i\ell(\phi^L(\cdots \phi^2(W^1\phi^1(z^0_i) + b^1)\cdots ), t_i)\right]\nonumber\\&\hspace{5.25pc}\times\delta(z^0_i - W^0 x_i - b^0)\Bigg)\exp\left[ - \sum_{l=0}^L\left(\frac{N_l}{2\sigma_w^2}W^l_{\alpha\beta}W^l_{\alpha\beta} + \frac1{2\sigma_b^2}b^l_\alpha b^l_\alpha\right)\right].
\end{align}
We can repeat this process iteratively until all of the pre-activations have been introduced. We find,
\begin{align}
Q&=\int[dW][db]\Bigg(\prod_{i\in\mathcal M}\prod_{l=0}^L\int[dz_i^l]\exp\left[-\frac{J_D}2\sum_i\ell(\phi^{L+1}(z^L_i), t_i)\right]\nonumber\\&\hspace{2pc}\times\prod_{l=0}^L\delta(z^l_i-W^l\phi^l(z^{l-1}_i) + b^l)\Bigg)\exp\left[ - \sum_{l=0}^L\left(\frac{N_l}{2\sigma_w^2}W^l_{\alpha\beta}W^l_{\alpha\beta} + \frac1{2\sigma_b^2}b^l_\alpha b^l_\alpha\right)\right]
\end{align}
where we will use $\phi^0(z^{-1})\equiv x$ interchangeably for notational simplicity. This procedure has essentially used a change of variables to introduce the pre-activations explicitly into the partition function. 

Here, $\delta$-functions constrain the pre-activations their correct values given the weights. To complete the proof we leverage the Fourier representation of the $\delta$-function as $\delta(x) = \int d\lambda e^{-ix\lambda}$. In particular we use Fourier space denoted by $\lambda^l_\alpha$ for each pre-activation constraint. We therefore find,
\begin{align}
Q&=\int[dW][db]\Bigg(\prod_{i\in\mathcal M}\prod_{l=0}^L\int[dz_i^l]\exp\left[-\frac{J_D}2\sum_i\ell(\phi^{L+1}(z^L_i), t_i)\right]\nonumber\\
&\hspace{2pc}\times\prod_{l=0}^L\delta(z^l_i-W^l\phi^l(z^{l-1}_i) + b^l)\Bigg)\exp\left[ - \sum_{l=0}^L\left(\frac{N_l}{2\sigma_w^2}W^l_{\alpha\beta}W^l_{\alpha\beta} + \frac1{2\sigma_b^2}b^l_\alpha b^l_\alpha\right)\right]\\
&=\int[dW][db]\Bigg(\prod_{i\in\mathcal M}\int[dz_i][d\lambda_i]\exp\left[-\frac{J_D}2\sum_i\ell(\phi^{L+1}(z^L_i), t_i)\right] \nonumber\\
&\hspace{6pc}\times \prod_{l=1}^L\exp\left[-i\lambda_{\alpha;i}^l(z^l_{\alpha;i} - W^l_{\alpha\beta}\phi(z^{l-1}_{\beta;i}) - b^l_\alpha)\right]\Bigg)\nonumber\\&\hspace{6pc}\times\exp\left[ - \sum_{l=0}^L\left(\frac{N_l}{2\sigma_w^2}W^l_{\alpha\beta}W^l_{\alpha\beta} + \frac1{2\sigma_b^2}b^l_\alpha b^l_\alpha\right)\right]\\
&=\int[d\Omega]\exp\Bigg[-\sum_{i\in\mathcal M}\Bigg\{\frac{J_D}2\ell(\phi^{L+1}(z^L_i),t_i) + i\lambda_{\alpha;i}^0(z^0_{\alpha;i} - W^0_{\alpha\beta}x_{\beta;i} - b^0_{\alpha;i})\nonumber  \\ &\hspace{1pc} + \sum_{l=1}^Li\lambda^l_{\alpha;i}(z^l_{\alpha;i}-W^l_{\alpha\beta}\phi^l(z_{\beta;i}^{l-1}) - b^l_\alpha)\Bigg\} -\sum_{l=0}^L\Bigg(\frac{N_l}{2\sigma_w^2}W^l_{\alpha\beta}W^l_{\alpha\beta} + \frac1{2\sigma_b^2}b^l_\alpha b^l_\alpha\Bigg)\Bigg]
\end{align}
as required.
\end{proof}

\begin{theorem}\label{sec:partition_function_feed_forward}
Provided $N_l\gg |\mathcal M|$, the weights, biases, and fictitious fields can be integrated out of eq.~\eqref{seq:partition_function_fict} to give a stochastic process involving only the pre-activations as,
\begin{align}
Q = \int[dz]\exp\left[-\frac{J_D}2\sum_{i\in\mathcal M}\ell(\phi^{L+1}(z^L_i),t_i) - \frac12\sum_{l=0}^L\left((\bm z^l_\alpha)^T(\bm \Sigma^l)^{-1}\bm z^l_\alpha + \ln|\bm\Sigma^l|\right)\right]
\label{seq:stochastic_partition_full}
\end{align}
where $(\bm z^l_\alpha)^T = (z^l_{\alpha;1},\cdots, z^l_{\alpha,|\mathcal M|})$ is a vector of pre-activations corresponding to each input to the network, $\bm\Sigma^l_{ij} = \sigma_w^2N_l^{-1}\phi^l(z^{l-1}_{\alpha;i})\phi^l(z^{l-1}_{\alpha;j}) + \sigma_b^2$ if $l>0$, and $\bm\Sigma^0_{ij} = \sigma_w^2N_l^{-1}x_{\alpha;i}x_{\alpha;j} + \sigma_b^2$ is the correlation matrix between activations of the network from different inputs.
\end{theorem}
\begin{proof}
We proceed directly completing the square and integrating out Gaussian variables. For notational simplicity we temporarily let $z^{-1} = x$ and $\phi^{0}(z) = z$ be linear. We then integrate out the weights and biases by completing the square,
\begin{align}
Q &= \int[d\Omega]\exp\Bigg[-\sum_{i\in\mathcal M}\Bigg\{\frac{J_D}2\ell(\phi^{L+1}(z^L_i),t_i) + \sum_{l=0}^Li\lambda^l_{\alpha;i}(z^l_{\alpha;i}-W^l_{\alpha\beta}\phi^l(z_{\beta;i}^{l-1}) - b^l_\alpha)\Bigg\}\nonumber\\&\hspace{5.5pc}-\sum_{l=0}^L\Bigg(\frac{N_l}{2\sigma_w^2}W^l_{\alpha\beta}W^l_{\alpha\beta} + \frac1{2\sigma_b^2}b^l_\alpha b^l_\alpha\Bigg)\Bigg]\\
&= \int[d\Omega]\exp\Bigg[-\frac{J_D}2\sum_{i\in\mathcal M}\ell(\phi^{L+1}(z^L_i),t_i) - \sum_{l=0}^L\sum_{i\in\mathcal M}i\lambda^l_{\alpha;i}z^l_{\alpha;i}\nonumber\\
&\hspace{2pc} - \sum_{l=0}^L\frac{N_l}{2\sigma_w^2}\left(W^l_{\alpha\beta} - \frac{i\sigma_w^2}{N_l}\sum_{i\in\mathcal M}\lambda^l_{\alpha;i}\phi^l(z^{l-1}_{\beta;i})\right)\left(W^l_{\alpha\beta} - \frac{i\sigma_w^2}{N_l}\sum_{i\in\mathcal M}\lambda^l_{\alpha;i}\phi^l(z^{l-1}_{\beta;i})\right)\nonumber\\
&\hspace{2pc} - \sum_{l=0}^L\frac{1}{2\sigma_b^2}\left(b^l_{\alpha} - i\sigma_b^2\sum_{i\in\mathcal M}\lambda^l_{\alpha;i}\right)\left(b^l_{\alpha} - i\sigma_b^2\sum_{i\in\mathcal M}\lambda^l_{\alpha;i}\right)\nonumber\\
&\hspace{2pc}-\sum_{l=0}^L\sum_{i,j\in\mathcal M}\lambda_{\alpha;i}^l\lambda_{\alpha;j}^l(\sigma_w^2N_l^{-1}\phi^l(z^{l-1}_{\beta;i})\phi^l(z^{l-1}_{\beta;j}) + \sigma_b^2)\Bigg]\\
&= \int[dz][d\lambda]\exp\Bigg[-\frac{J_D}2\sum_{i\in\mathcal M}\ell(\phi^{L+1}(z^L_i),t_i) - \sum_{l=0}^L\sum_{i\in\mathcal M}i\lambda^l_{\alpha;i}z^l_{\alpha;i}\nonumber\\&\hspace{7pc}-\frac12\sum_{l=0}^L\sum_{i,j\in\mathcal M}\lambda_{\alpha;i}^l\lambda_{\alpha;j}^l(\sigma_w^2N_l^{-1}\phi^l(z^{l-1}_{\beta;i})\phi^l(z^{l-1}_{\beta;j}) + \sigma_b^2)\Bigg].\label{eq:partition_function_constraints}
\end{align}
Interestingly, we notice that upon integrating out the weights and biases, the pre-activations from different inputs become coupled. This is reminiscent of replica calculations in the spin glass literature.

We now rewrite the above expression to elucidate its structure. To do this we first let $(\bm\lambda^l_\alpha)^T = (\lambda^l_{\alpha;1}, \lambda^l_{\alpha;2}, \cdots, \lambda^l_{\alpha;|\mathcal M|})$, $(\bm z^l_\alpha)^T = (z^l_{\alpha;1}, z^l_{\alpha;2},\cdots z^l_{\alpha;|\mathcal M|})$, and $(\bm\phi^l_\alpha)^T = (\phi^l(z^{l-1}_{\alpha;1}),\phi^l(z^{l-1}_{\alpha;2}),\cdots,\phi^l(z^{l-1}_{\alpha;|\mathcal M|}))$. Finally we define the matrix $\bm \Sigma^l = \sigma_w^2N_l^{-1}\bm \phi^l_\alpha(\bm\phi^l_\alpha)^T + \bm 1\sigma_b^2$ where $\bm 1$ is the $|\mathcal M|\times|\mathcal M|$ matrix of ones. Using this notation we can rewrite eq.~\eqref{eq:partition_function_constraints} as,
\begin{equation}
Q = \int[dz][d\lambda]\exp\Bigg[-\frac{J_D}2\sum_{i\in\mathcal M}\ell(\phi^{L+1}(z^L_i),t_i) - \sum_{l=0}^L\left\{\frac12(\bm\lambda^l_\alpha)^T\bm \Sigma^l\bm\lambda^l_{\alpha} - i(\bm\lambda^l_\alpha)^T\bm z^l_\alpha\right\}\Bigg].\label{eq:partition_function_constraints_rewrite}
\end{equation}
Eq.~\eqref{eq:partition_function_constraints_rewrite} clearly has the structure of a multivariate Gaussian as a function of the $\bm\lambda^l_\alpha$. In principle it is therefore possible to integrate out the $\bm\lambda^l_\alpha$. We notice, however, that $\bm\Sigma^l$ is an $|\mathcal M|\times |\mathcal M|$ matrix constructed as a sum of $N_l+1$ terms each being the outer-product of a vector. It follows that the rank of $\bm\Sigma^l$ is at most $N_l + 1$. For this work we will be explicitly interested in the large $N_l$ limit and so we may safely assume that $\bm\Sigma^l$ is full-rank. However, more care must be taken when $N_l+1 \sim |\mathcal M|$. 

Thus, in the case that $N_l\gg |\mathcal M|$ we may integrate out the $\bm\lambda^l_\alpha$ in the usual way to find,
\begin{equation}
Q = \int[dz]\exp\left[-\frac{J_D}2\sum_{i\in\mathcal M}\ell(\phi^{L+1}(z^L_i),t_i) - \frac12\sum_{l=0}^L\left((\bm z^l_\alpha)^T(\bm \Sigma^l)^{-1}\bm z^l_\alpha + \ln|\bm\Sigma^l|\right)\right]
\end{equation}
as required.
\end{proof}

\begin{corollary}
In the event that the network has only a single input eq.~\eqref{eq:stochastic_partition_full} reduces to,
\begin{align}
Q = \int[dz]\exp\Bigg[-\frac12\Bigg\{J_D\ell(\phi^{L+1}(z^L),t) &+ \frac{z_\alpha^0z_\alpha^0}{\sigma_w^2N_0^{-1}x_\beta x_\beta + \sigma_b^2} \nonumber\\
&+ \sum_{l=1}^L\frac{z_\alpha^lz_\alpha^l}{\sigma_w^2N_l^{-1}\phi^l(z^{l-1}_\beta)\phi^l(z^{l-1}_\beta) + \sigma_b^2}\nonumber\\
&+ \sum_{l=1}^LN_l\log(\sigma_w^2N_l^{-1}\phi^l(z^l_\alpha)\phi^l(z^l_\alpha) + \sigma_b^2) \Bigg\}\Bigg].\label{seq:stochastic_partition_full_single_input}
\end{align}
Here we omit the sample index since it is unnecessary.
\end{corollary}
\begin{proof}
This result follows directly from the previous result by plugging in for only a single input.
\end{proof}

\subsection{Theoretical Results on Linear Stochastic Networks}\label{app:linear_network}

Here we prove several results elucidating the behavior of the linear stochastic network on a ring. We will begin with the full partition function for the linear stochastic network,
\begin{equation}
Q  = \int[dz]\exp\left[-\frac12 \sum_{l=0}^L\left\{\frac{z^l_\alpha z^l_\alpha}{\sigma_w^2N^{-1}z^{l-1}_\beta z^{l-1}_\beta + \sigma_b^2} + N\log(\sigma_w^2 N^{-1}z^l_\alpha z^l_\alpha + \sigma_b^2)\right\}\right].\label{eq:linear_partition_full}
\end{equation}
Our first result concerns the change of variables into hyperspherical coordinates. We will denote the radius to be $r^l$.

\begin{result}
The energy for the stochastic linear network on a ring can be changed into hyperspherical coordinates. The resulting lattice model is described by the energy,
\begin{equation}
\mathcal L(\{r^l\}) = \frac12\sum_{l=0}^L\left\{\frac{(r^l)^2}{\sigma_w^2N^{-1}(r^{l-1})^2 + \sigma_b^2} - N\log\left(\frac{(r^l)^2}{\sigma_w^2N^{-1}(r^l)^2 + \sigma_b^2}\right)\right\}\label{eq:energy_linear_stochastic}
\end{equation}
where $r^l$ is the norm of the pre-activation in layer $l$.
\end{result}
\begin{proof}
We proceed by simply making the change of variables in eq.~\eqref{eq:linear_partition_full}. Since the integrand is isotropic we express the integral over angles in layer $l$ by $d\Omega^l$. However we note that the angular integrals will change the partition function by at most a constant and so may be discarded.
\begin{align}
Q  &= \int[dz]\exp\left[-\frac12 \sum_{l=0}^L\left\{\frac{z^l_\alpha z^l_\alpha}{\sigma_w^2N^{-1}z^{l-1}_\beta z^{l-1}_\beta + \sigma_b^2} + N\log(\sigma_w^2 N^{-1}z^l_\alpha z^l_\alpha + \sigma_b^2)\right\}\right]\\
&= \prod_{l=0}^L\int[dz^l]\exp\left[-\frac12 \sum_{l=0}^L\left\{\frac{z^l_\alpha z^l_\alpha}{\sigma_w^2N^{-1}z^{l-1}_\beta z^{l-1}_\beta + \sigma_b^2} + N\log(\sigma_w^2 N^{-1}z^l_\alpha z^l_\alpha + \sigma_b^2)\right\}\right]\\
&=  \prod_{l=0}^L\int dr^ld\Omega^l(r^l)^{N-1}\exp\Bigg[-\frac12 \sum_{l=0}^L\Bigg\{\frac{(r^l)^2}{\sigma_w^2N^{-1}(r^{l-1})^2 + \sigma_b^2}\nonumber\\ 
&\hspace{14pc} + N\log(\sigma_w^2 N^{-1}(r^l)^2+ \sigma_b^2)\Bigg\}\Bigg]\\
&\approx \int [dr] exp\left[-\frac12 \sum_{l=0}^L\left\{\frac{(r^l)^2}{\sigma_w^2N^{-1}(r^{l-1})^2 + \sigma_b^2} - N\log\left(\frac{(r^l)^2}{\sigma_w^2 N^{-1}(r^l)^2 + \sigma_b^2}\right) \right\}\right].\label{eq:linear_partition_norm}
\end{align} 
The definition of the energy follows immediately. Here we replace $N-1$ by $N$ for convenience since typically $N\gg 1$.
\end{proof}

We now discuss the saddle point approximation to eq.~\eqref{eq:linear_partition_norm}. We begin our discussion by that when $\sigma_w^2 < 1$, eq.~\eqref{eq:energy_linear_stochastic} is minimized by a uniform arrangement of spins, $r^l = r^*$ independent of layer.

\begin{result}
When $\sigma_w^2 < 1$, there exists a constant configuration of spins, with $r^l = r^*$ independent of layer, that minimizes the energy for the stochastic neural network on a ring, given by eq~\eqref{eq:energy_linear_stochastic}. The constant solution is given by,
\begin{equation}
r^* = \sqrt{\frac{N\sigma_b^2}{\Delta_w}}\label{eq:linear_stochastic_fixed_point}
\end{equation}
where $\Delta_w = 1-\sigma_w^2$.
\end{result}
\begin{proof}
When $r^l = r^*$ independent of layer, eq.~\eqref{eq:energy_linear_stochastic} will be given by,
\begin{equation}
\mathcal L(r^*) = \frac L2\left[\frac{(r^*)^2}{\sigma_w^2N^{-1}(r^*)^2 + \sigma_b^2} - N \log\left(\frac{(r^*)^2}{\sigma_w^2N^{-1}(r^*)^2 + \sigma_b^2}\right)\right].\label{eq:energy_linear_stochastic_uniform}
\end{equation}
Note that this equation has the form $x-N\log x$ which has a minimum when $x = N$. It follows that eq.~\eqref{eq:energy_linear_stochastic_uniform} will have a minimum precisely when
\begin{equation}
(r^*)^2 = \sigma_w^2(r^*)^2 + N\sigma_b^2
\end{equation}
as required.
\end{proof}

Next we can expand eq.~\eqref{eq:energy_linear_stochastic} in small nonuniform fluctuations about $r^*$. 
\begin{result}
Small fluctuations about eq.~\eqref{eq:linear_stochastic_fixed_point}, given by $r^l = r^* + \epsilon^l$, are governed by the energy,
\begin{equation}
\mathcal L(\{r^*+\epsilon^l\}) = \mathcal L(r^*) + \frac{\Delta_w}{\sigma_b^2}\sum_{l=0}^L\left[(1+\sigma_w^4)(\epsilon^l)^2 - 2\sigma_w^2\epsilon^l \epsilon^{l-1}\right] + \mathcal O(\epsilon^4). \label{eq:energy_linear_stochastic_fluctuations}
\end{equation}
\end{result} 
\begin{proof}
We consider the cost of small fluctuations about the constant solution so that $r^l = r^*+\epsilon^l$ where $\epsilon^l\ll r^*$. For notational simplicity we write $D = \sigma_w^2N^{-1}(r^*)^2+\sigma_b^2$ and $\alpha = \sigma_w^2N^{-1}$. We then expand the perturbation to the energy to quadratic order about $r^*$ to find,
\begin{align}
\mathcal L(\{r^* + \epsilon^l\}) &= \sum_{l=0}^L\left(\frac12\frac{(r^l)^2}{\sigma_w^2N^{-1}(r^{l-1})^2+\sigma_b^2} - N\log r^l+\frac N2\log(\sigma_w^2N^{-1}(r^l)^2+\sigma_b^2)\right)\\
&= \sum_{l=0}^L\Bigg\{\frac12\frac{(r^*)^2+2(r^*)\epsilon^l+(\epsilon^l)^2}{\sigma_w^2N^{-1}(r^*)^2+\sigma_b^2+\sigma_w^2N^{-1}\epsilon^{l-1}(2(r^*)+\epsilon^{l-1})}-N\log (r^*)\nonumber\\
&\hspace{3pc} - N\log\left(1+\frac{\epsilon^l}{(r^*)}\right)+ \frac N2\log\left(\sigma_w^2N^{-1}(r^*)^2+\sigma_b^2\right)\nonumber\\
&\hspace{3pc} + \frac N2\log\left(1+\frac{\sigma_w^2N^{-1}\epsilon^l(2(r^*)+\epsilon^l)}{\sigma_w^2N^{-1}(r^*)^2+\sigma_b^2}\right)\Bigg\}\\
&= \mathcal L(r^*) + \sum_{l=0}^L\Bigg\{-\frac\alpha{D^2}(r^*)^3\epsilon^{l-1} + \left(\frac {r^*}D - \frac N{r^*} + N\frac\alpha D r^*\right)\epsilon^l \nonumber\\
&\hspace{6.5pc} + \frac{\alpha z^2}{2D^2}\left(\frac{4\alpha}D(r^*)^2-1\right)(\epsilon^{l-1})^2 - 2\frac{\alpha (r^*)^2}{D^2}\epsilon^l\epsilon^{l-1}\nonumber \\ 
&\hspace{6.5pc} + \left(\frac 1{2D}+\frac{N}{2(r^*)^2} + \frac{N\alpha}{2D} -\frac{N\alpha^2}{D^2}(r^*)^2\right)(\epsilon^l)^2\Bigg\}.
\end{align}
Next we substitute in for $r^*$ noting that $D = (r^*)^2/N$. It follows that,
\begin{align}
\mathcal L(\{r^*+\epsilon^l\})-\mathcal L(r^*) &= \sum_{l=0}^L\Bigg\{-\frac{\alpha N^2}{r^*}\epsilon^{l-1} + \frac{\alpha N^2}{r^*}\epsilon^l + \frac{\alpha N^2}{2(r^*)^2}(4N\alpha - 1)(\epsilon^{l-1})^2  \nonumber\\
&\hspace{3pc}- 2\frac{\alpha N^2}{(r^*)^2}\epsilon^l\epsilon^{l-1}+ \left(\frac{N}{(r^*)^2} + \frac{N^2\alpha}{2(r^*)^2} - \frac{N^3\alpha^2}{(r^*)^2}\right)(\epsilon^l)^2\Bigg\}.
\end{align}
We note that each term in the sum appears twice - once from the $l$ term and once from the $l+1$ term - except for $\epsilon^l\epsilon^{l-1}$. We may therefore reorganize the sum to symmetrize the different pieces. As a result we note that all the terms linear in $\epsilon^l$ vanish. Substituting in for $z^*$ we find,
\begin{align}
\mathcal L(\{r^* + \epsilon^l\})-\mathcal L(r^*) &=\sum_{l=0}^L\Bigg\{\left(\frac{N}{(r^*)^2}+\frac{N^3\alpha^2}{(r^*)^2}\right)(\epsilon^l)^2 - 2\frac{\alpha N^2}{(r^*)^2}\epsilon^l\epsilon^{l-1}\Bigg\}\\
&=\frac{\Delta_w}{\sigma_b^2}\sum_{l=0}^L\left[(1+\sigma_w^4)(\epsilon^l)^2 - 2\sigma_w^2\epsilon^l\epsilon^{l-1}\right]
\end{align}
as required. 
\end{proof}
Examining eq.~\eqref{eq:energy_linear_stochastic_fluctuations} we note, among other things, that as $\sigma_w^2\to 1$ the cost of fluctuations goes to zero. We have successfully constructed a linear field theory for small fluctuations in the stochastic linear network for $\sigma_w^2 <1$. It is important to note that if it were desirable one could continue the expansion to higher order. This would give you perturbative corrections to the linear theory that we expect to be $\mathcal O(N^{-1})$. One could imagine using this expansion to study the effect of finite width networks.

Next we show that eq.~\eqref{eq:energy_linear_stochastic_fluctuations} can be diagonalized by switching to Fourier basis. Because our network is topologically equivalent to a ring we can always expand $\epsilon^l$ in Fourier series to get,
\begin{equation}
\epsilon^l = \sum_q \epsilon_q e^{-iql}.
\end{equation}
Since $\epsilon^{L+1} = \epsilon^0$ it follows that $q = 2n\pi/(L+1)$ for $n\in\mathbb Z$. The depth of our network therefore determines the highest frequency fluctuations that we will be able to observe.

\begin{result}
Replacing $\epsilon^l$ in eq.~\eqref{eq:energy_linear_stochastic_fluctuations} by its Fourier series we get an energy,
\begin{equation}
\mathcal L(\{r^* + \epsilon_q\}) - \mathcal L(r^*) = \frac{L\Delta_w}{\sigma_b^2}\sum_q\left\{(1+\sigma_w^4)-2\sigma_w^2\cos q\right\}|\epsilon_q|^2\label{eq:energy_linear_stochastic_fluctuations_fourier}. 
\end{equation}
The associated probability distribution is factorial Gaussian. It follows that the different Fourier modes of fluctuations in a deep linear network behave as uncoupled Gaussian random variables.
\end{result}
\begin{proof}
It follows that we may rewrite eq.~\ref{eq:energy_linear_stochastic_fluctuations} as,
\begin{align}
\mathcal L(\{r^* + \epsilon^l\}) - \mathcal L(r^*) &= \frac{\Delta_w}{\sigma_b^2}\sum_{l=0}^L\left((1+\sigma_w^4)(\epsilon^l)^2 - 2\sigma_w^2\epsilon^l\epsilon^{l-1}\right)\\
&=\frac{\Delta_w}{\sigma_b^2}\sum_{l=0}^L\sum_{qq'}\left((1+\sigma_w^4)\epsilon_q\epsilon_{q'}e^{-i(q+q')l} - 2\sigma_w^2\epsilon_q\epsilon_{q'}e^{-i(q+q')l}e^{iq}\right)\\
&=\frac{\Delta_w}{\sigma_b^2}\sum_{qq'}\left((1+\sigma_w^4)\epsilon_q\epsilon_{q'}- 2\sigma_w^2\epsilon_q\epsilon_{q'}e^{iq}\right)\sum_{l=0}^Le^{-i(q+q')l}\\
&=\frac{\Delta_w}{\sigma_b^2}\sum_{qq'}\left((1+\sigma_w^4)\epsilon_q\epsilon_{q'}- 2\sigma_w^2\epsilon_q\epsilon_{q'}e^{iq}\right)\delta_{q,-q'}\\
&=\frac{L\Delta_w}{\sigma_b^2}\sum_{q}\left((1+\sigma_w^4)\epsilon_q\epsilon_{-q}- 2\sigma_w^2\epsilon_q\epsilon_{-q}e^{iq}\right)
\end{align}
where we have used the exponential representation of the $\delta$-function,
\begin{equation}
\sum_le^{-iql} = L\delta_{q,0}.
\end{equation}
Finally note that since $\epsilon^l$ is real it must be true that $\epsilon_{-q} = \epsilon_q^\dag$. It follows that,
\begin{align}
\mathcal L(\{r^* + \epsilon^l\}) - \mathcal L(r^*) &= \frac{L\Delta_w}{\sigma_b^2}\sum_{q}\left((1+\sigma_w^4)\epsilon_q\epsilon_{-q}- 2\sigma_w^2\epsilon_q\epsilon_{-q}e^{iq}\right)\\
&=\frac{L\Delta_w}{\sigma_b^2}\sum_{q}\left((1+\sigma_w^4)- 2\sigma_w^2e^{iq}\right)|\epsilon_q|^2\\
&=\frac{L\Delta_w}{\sigma_b^2}\sum_{q}\left((1+\sigma_w^4)- 2\sigma_w^2\cos q\right)|\epsilon_q|^2
\end{align} 
where in the last step we have rearranged the sum to pair each mode with its complex conjugate.
\end{proof}

The final theoretical result for this section shows that the long-distance behavior of the linear stochastic network can be well described by an effective field theory.

\begin{result}
Long range fluctuations of the stochastic linear network (i.e. fluctuations in which $\epsilon^l$ varies slowly on the scale of one layer) are governed by the effective field theory defined by the energy,
\begin{equation}
U[\epsilon(x)] = \frac{\Delta_w}{\sigma_b^2}\int dx\left[\Delta_w^2(\epsilon(x))^2 + \sigma_w^2\left(\frac{\partial \epsilon(x)}{\partial x}\right)^2\right].
\end{equation}
\end{result}
\begin{proof}
Note that we can rewrite eq.~\eqref{eq:energy_linear_stochastic_fluctuations} as,
\begin{align}
U(\{\epsilon^l\}) &= \frac{\Delta_w}{\sigma_b^2}\sum_{l=0}^L\left[(1+\sigma_w^4)(\epsilon^l)^2 - 2\sigma_w^2\epsilon^l \epsilon^{l-1}\right]\\
&=\frac{\Delta_w}{\sigma_b^2}\sum_{l=0}^L\left[(1-2\sigma_w^2+\sigma_w^4)(\epsilon^l)^2 + \sigma_w^2(\epsilon^l - \epsilon^{l-1})^2\right]\\
&=\frac{\Delta_w}{\sigma_b^2}\sum_{l=0}^L\left[\Delta_w^2(\epsilon^l)^2 + \sigma_w^2(\epsilon^l - \epsilon^{l-1})^2\right].
\end{align}
Let us now suggestively write $\epsilon(l) = \epsilon^l$. If $\epsilon^l$ is varying slowly on the scale of individual layers and further if $L\gg 1$ then we can approximate,
\begin{equation}
\frac{\epsilon^l - \epsilon^{l-1}}{1} \approx \frac{\partial \epsilon(l)}{\partial l}.
\end{equation}
We can additionally interpret the sum over layers as a Riemann sum. This yields the effective field theory for long-wavelength fluctuations,
\begin{equation}
U[\epsilon(x)] \approx \frac{\Delta_w}{\sigma_b^2}\int dx\left[\Delta_w^2\epsilon^2(x) + \sigma_w^2\left(\frac{\partial \epsilon}{\partial x}\right)^2\right]
\end{equation}
with the replacement $l\to x$.
\end{proof}

\subsection{Numerical Results on Linear Stochastic Networks}\label{app:linear_network_numerics}

We now provide a more detailed description of the numerical methods discussed in the main text. We would like to sample the mean pre-activation and fluctuations about the mean for deep and wide linear stochastic networks on a ring. However, in practice it is easier to consider a linear topology (non-ring) and consider the ``bulk'' fluctuations and mean of the pre-activations after any transient from the input to the network has decayed. In general we consider constant width networks with $N = 200$ and $L = 1024$. 

To generate a single sample of pre-activations from the ensemble of linear stochastic networks with $J_D = 0$, we randomly initialize the weights and biases according to $W_{ij}^l\sim\mathcal N(0,\sigma_w^2/N)$ and $b^l_i\sim\mathcal N(0,\sigma_b^2)$. We then feed a random input into the network and record the norm of the pre-activations after each layer. By repeating this process $M = 200$ times we are able to get Monte-Carlo samples from the ensemble of pre-activations for linear stochastic neural networks. 
\begin{figure}[!h]                                                              
\includegraphics[width=0.43\textwidth]{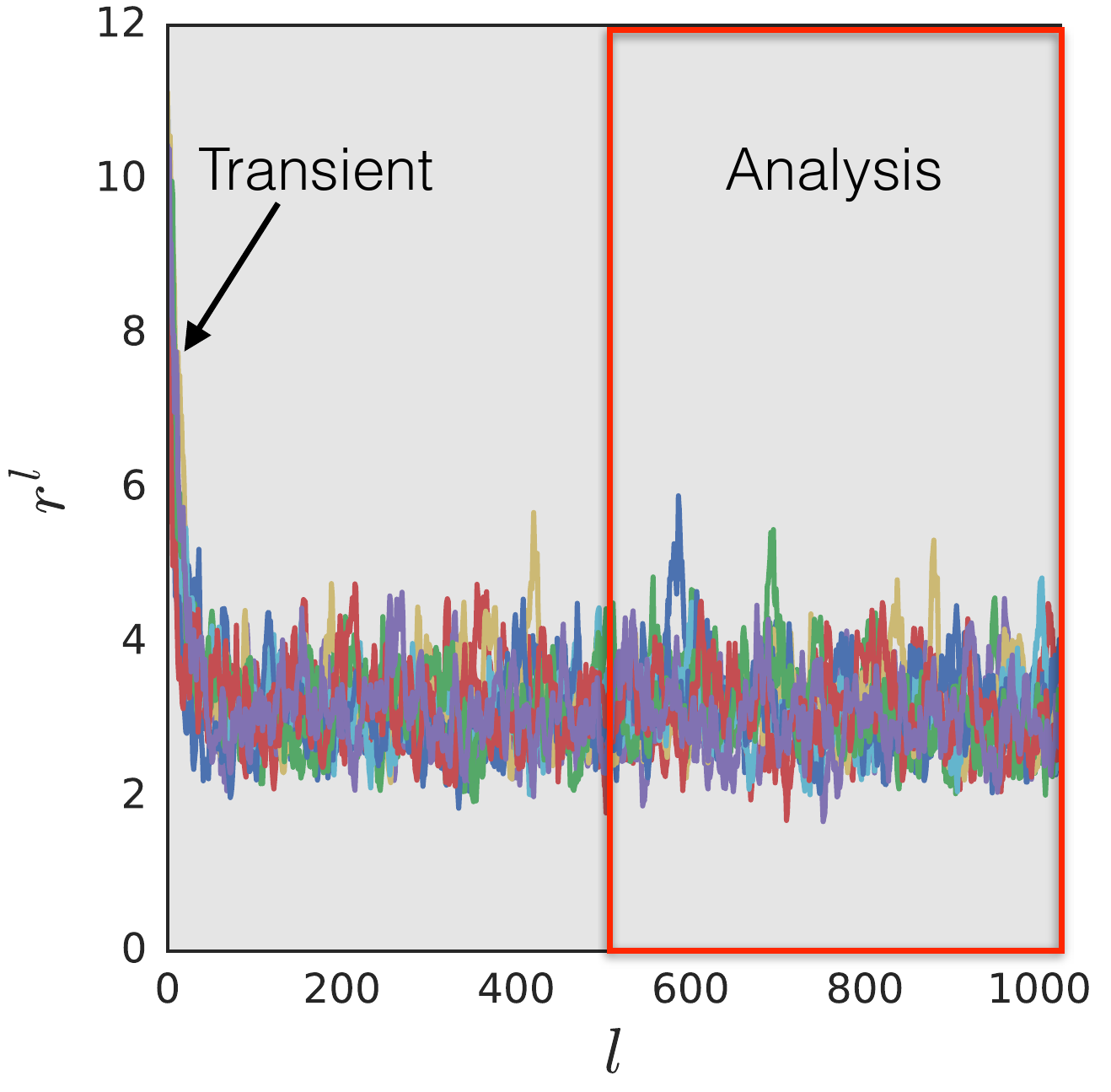} \hspace{3pc} \includegraphics[width=0.43\textwidth]{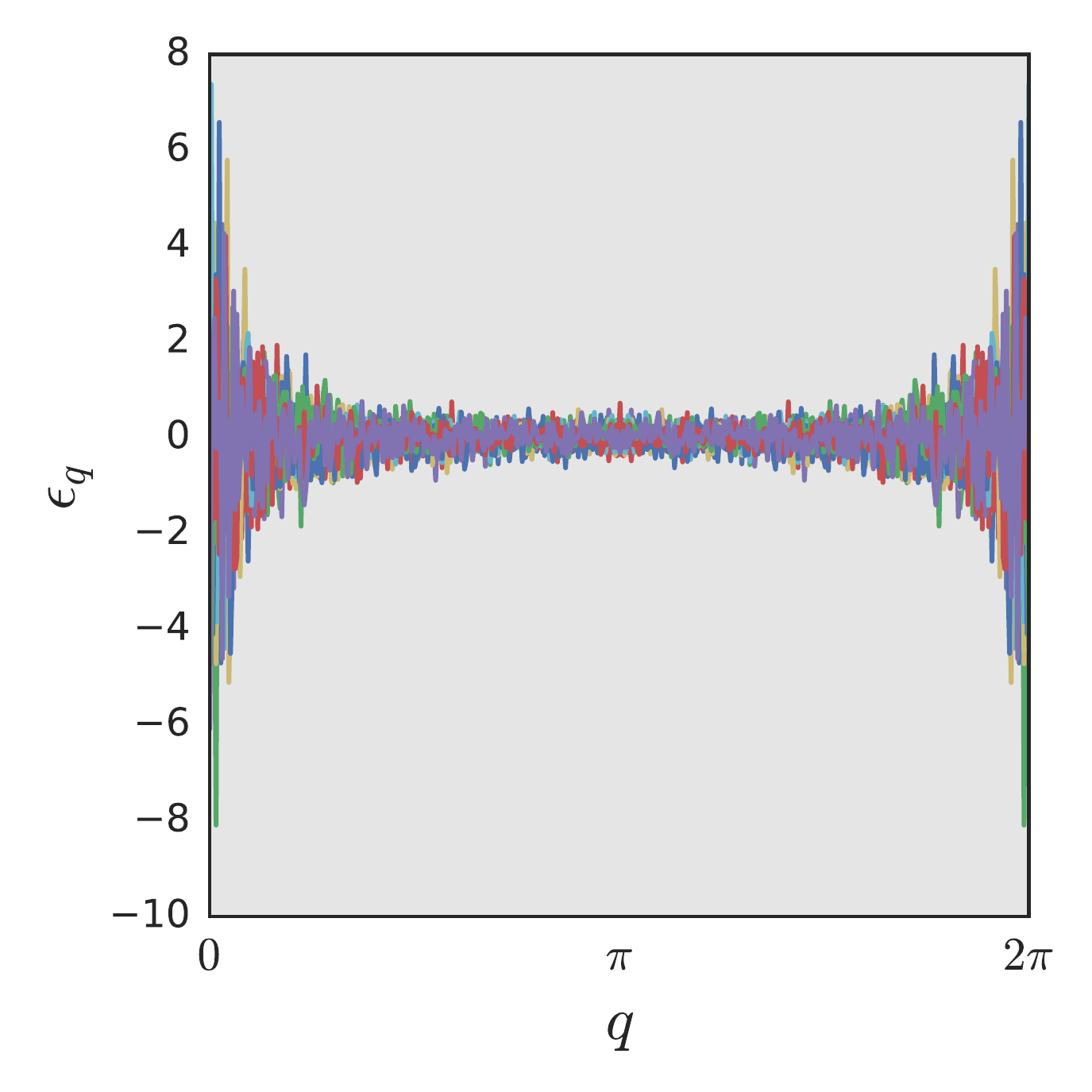}                     
\centering                                                                      
\caption{Samples from an $L=1024$ layer stochastic linear neural network with $J_D = 0$, $N = 200$, $\sigma_b^2 = 0.01$, and $\sigma_w^2 = 0.9$. (a) Real space values of the norm as it travels through different instantiations of the network. We notice that early in the network there is a transient signal after which the pre-activations reach their ``bulk'' behaviour. We separate the signal propagation into two regions: a transient region and an analysis region. (b) The Fast Fourier Transform of the pre-activations in the analysis region of (a) for different instantiations of the network.}
\label{fig:app_stochastic_network_samples}                                          
\end{figure}   
Fig.~\ref{fig:app_stochastic_network_samples} (a) shows the norm of the pre-activations at different layers of a stochastic linear network for different random instantiations of the weights. This plot therefore shows different samples from the ensemble of stochastic linear networks. We notice that there is a transient effect of the input that lasts for around 100 layers. To perform our analysis and make a correspondence between the stochastic network on a ring we would like to only consider the ``bulk'' behaviour of the network. To this end we divide the trajectory of the norm of pre-activations into two halves and study only the half from layer 512 to layer 1024. We anticipate for all values of $\sigma_w^2$ studied the transient ought to have decayed by this point.

In Fig.~\ref{fig:app_stochastic_network_samples} (b) we show the Fast Fourier Transform (FFT) for the fluctuations of the instantiations of the norm of the pre-activations about the theoretical mean. We notice that the Fourier modes - as with the real space fluctuations, the Fourier modes are also stochastic. There is clearly a change in the variance of the modes as a function of wavevector.
\begin{figure}[!h]                                                              
\includegraphics[width=0.6\textwidth]{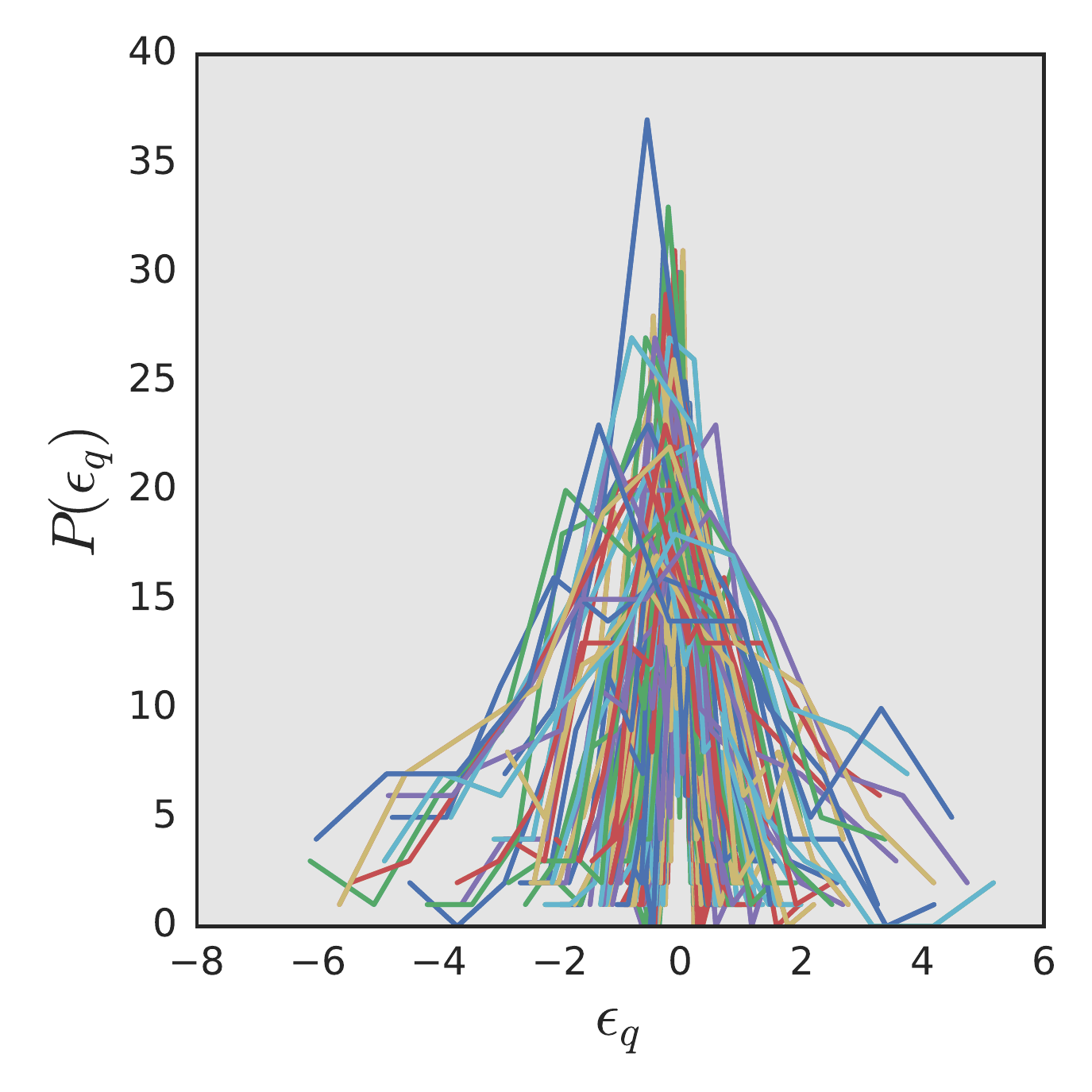}                     
\centering                                                                      
\caption{The distribution of small fluctuations, $\epsilon_q$. Different curves represent different values of $q$. }
\label{fig:stochastic_network_histogram}                                          
\end{figure} 
Fig.~\ref{fig:stochastic_network_histogram} shows histograms of $\epsilon_q$ for different values of $q$.  
\begin{figure}[!h]                                                              
\includegraphics[width=0.43\textwidth]{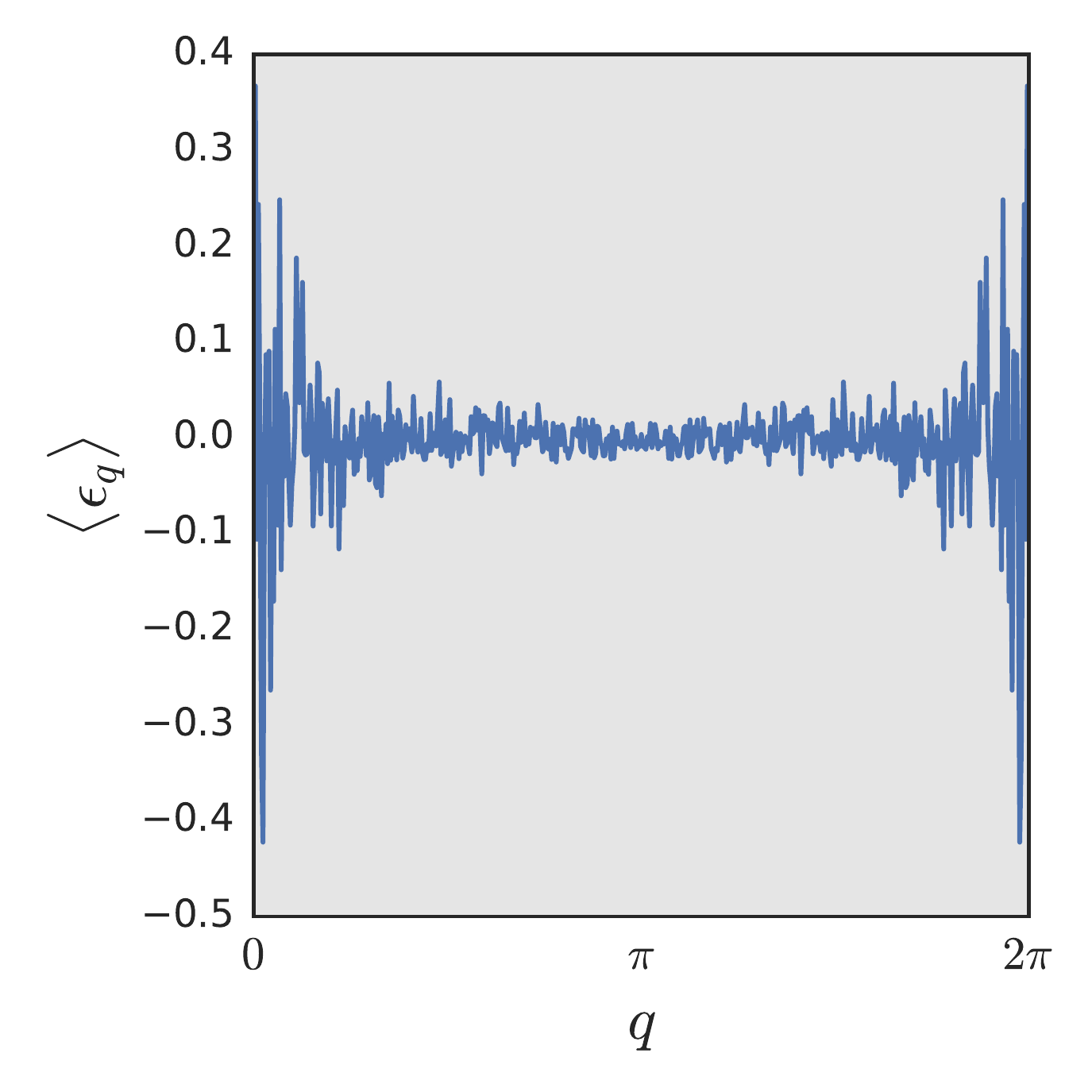} \hspace{3pc} \includegraphics[width=0.43\textwidth]{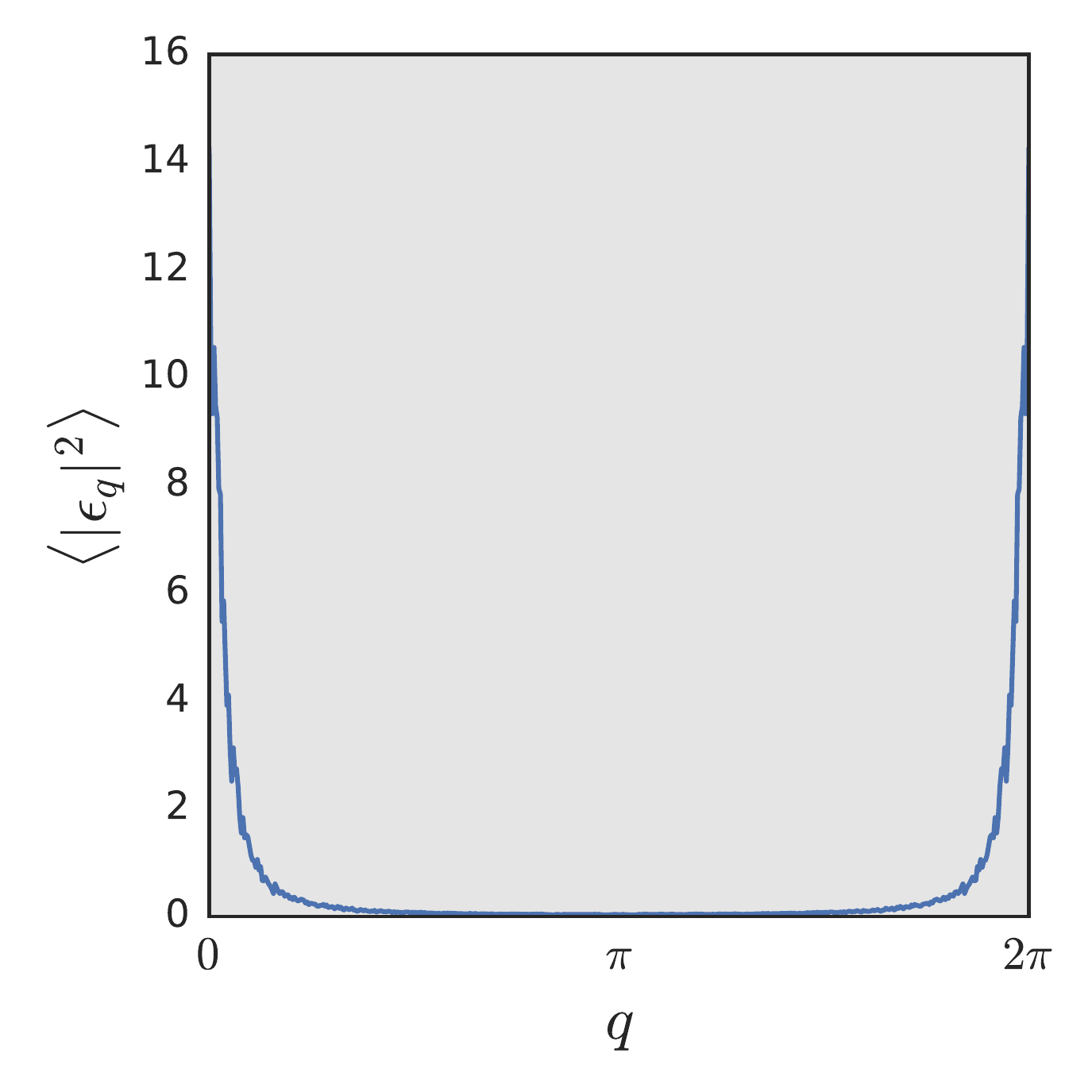}                     
\centering                                                                      
\caption{Statistics of small fluctuations. (a) Shows the expected value of $\epsilon_q$ and (b) shows the variance, $\langle\epsilon_q^\dag \epsilon_q\rangle $.}
\label{fig:stochastic_network_mean_variance}                                          
\end{figure} 
Fig.~\ref{fig:stochastic_network_mean_variance} (a) shows the mean of the distribution of $\epsilon_q$ as a function of wavevector. It is clear that $\epsilon_q$ has mean zero everywhere but the uncertainty in the measurement of the mean increases as $q\to 0$.  In fig.~\ref{fig:stochastic_network_mean_variance} (b) we finally plot the variance of small fluctuations. This is what is plotted against the theoretical prediction in the main text.

\subsection{Theoretical Results on Rectified Linear Stochastic Networks}\label{app:rectified_linear_network}

Next we discuss the case of a rectified linear network. For rectified linear units we begin by noting that any vector $z^l$ can be decomposed into its positive and negative components as $z^l = z^l_+ + z^l_-$. With this decomposition in mind we can write the squared norm of $z^l$ as $|z^l|^2 = |z^l_-|^2 + |z^l_+|^2$. The partition function can therefore be written as
\begin{equation}\label{eq:partition_relu_raw}
Q = \int [dz]\exp\left[-\frac12\sum_{l=0}^L\left(\frac{|z_+^l|^2 + |z^l_-|^2}{\sigma_w^2N^{-1}|z^{l-1}_+|^2 + \sigma_b^2} + N\log(\sigma_w^2N^{-1}|z^l_+|^2+\sigma_b^2)\right)\right].
\end{equation}
We wish to perform a change of variables into hyperspherical coordinates as before. Unlike in the linear case, here we must be more careful since the probability distribution is anisotropic. This leads us to our first result.

\begin{result}
The partition function for rectified linear stochastic networks can bet written as,
\begin{align}
Q &= 2\left(\frac{\sqrt\pi}{2}\right)^N\prod_l \sum_{k_l=0}^N{N\choose k_l}\frac1{\Gamma\left(\frac{N-k_l}2\right)\Gamma\left(\frac{k_l}2\right)}\int dr^l_+dr^l_-(r^l_+)^{N-k_l-1}(r^l_-)^{k_l-1}\nonumber\\
&\hspace{3pc}\times\exp\left[-\frac12\sum_{l=0}^L\left(\frac{(r_+^l)^2 + (r^l_-)^2}{\sigma_w^2N^{-1}(r^{l-1}_+)^2 + \sigma_b^2} + N\log(\sigma_w^2N^{-1}(r^l_+)^2+\sigma_b^2)\right)\right]\label{eq:partition_function_relu_orthants}
\end{align}
by making a hyperspherical coordinate transformation in both $z^l_+$ and $z^l_-$ separately. Here $r_+^l$ is the norm of the positive components of the pre-activations in layer $l$, $r_-^l$ is the norm of the negative components of the pre-activations, and $k^l$ is the number of components of the pre-activations that are positive.
\end{result}
\begin{proof}
The partition function in eq~\eqref{eq:partition_relu_raw} can be decomposed into a sum over orthants for each layer separately as,
\begin{align}
Q &= \int\prod_l [dz^l]\exp\left[-\frac12\sum_{l=0}^L\left(\frac{|z_+^l|^2 + |z^l_-|^2}{\sigma_w^2N^{-1}|z^{l-1}_+|^2 + \sigma_b^2} + N\log(\sigma_w^2N^{-1}|z^l_+|+\sigma_b^2)\right)\right]\nonumber\\
&= \frac1{2^N}\prod_l \sum_{k_l=0}^N\int[dz^l_+][dz^l_-]{N\choose k_l}\exp\Bigg[-\frac12\sum_{l=0}^L\Bigg(\frac{|z_+^l|^2 + |z^l_-|^2}{\sigma_w^2N^{-1}|z^{l-1}_+|^2 + \sigma_b^2} \nonumber\\ 
&\hspace{14pc}+ N\log(\sigma_w^2N^{-1}|z^l_+|+\sigma_b^2)\Bigg)\Bigg]
\end{align}
where $N \choose k_l$ counts the number of orthants with $k_l$ positive components. In the above $z^l_+$ is integrated over $\mathbb R^{N-k_l}$ and $z^l_-$ is integrated over $\mathbb R^{k_l}$. 

In each orthant, the integrand is spherically symmetric over $z^l_+$ and $z^l_-$ separately. We may therefore make two change of variables into spherical coordinates for both $z^l_+$ and $z^l_-$ respectively.
\begin{align}
Q &= \frac1{2^N}\prod_l \sum_{k_l=0}^N\int[dz^l_+][dz^l_-]{N\choose k_l}\exp\Bigg[-\frac12\sum_{l=0}^L\Bigg(\frac{|z_+^l|^2 + |z^l_-|^2}{\sigma_w^2N^{-1}|z^{l-1}_+|^2 + \sigma_b^2} \nonumber\\ 
&\hspace{14pc}+ N\log(\sigma_w^2N^{-1}|z^l_+|+\sigma_b^2)\Bigg)\Bigg]\\
 &= \frac1{2^N}\prod_l \sum_{k_l=0}^N{N\choose k_l}\int dr^l_+ dr^l_-d\Omega_+d\Omega_- (r^l_+)^{N-k_l-1}(r^l_-)^{k_l-1}\nonumber\\ 
 &\hspace{3pc}\times\exp\left[-\frac12\sum_{l=0}^L\left(\frac{|z_+^l|^2 + |z^l_-|^2}{\sigma_w^2N^{-1}|z^{l-1}_+|^2 + \sigma_b^2} + N\log(\sigma_w^2N^{-1}|z^l_+|+\sigma_b^2)\right)\right]\\
 &= 2\left(\frac{\sqrt\pi}{2}\right)^N\prod_l \sum_{k_l=0}^N{N\choose k_l}\frac1{\Gamma\left(\frac{N-k_l}2\right)\Gamma\left(\frac{k_l}2\right)}\int dr^l_+dr^l_-(r^l_+)^{N-k_l-1}(r^l_-)^{k_l-1}\nonumber\\
&\hspace{3pc}\times\exp\left[-\frac12\sum_{l=0}^L\left(\frac{(r_+^l)^2 + (r^l_-)^2}{\sigma_w^2N^{-1}(r^{l-1}_+)^2 + \sigma_b^2} + N\log(\sigma_w^2N^{-1}(r^l_+)^2+\sigma_b^2)\right)\right]
\end{align}
Where $d\Omega_+$ and $d\Omega_-$ are angular integrals over the positive and negative components respectively. In the final step we have integrated over the angular piece explicitly to give a volume factor which, crucially, depends on $k_l$. 
\end{proof} 

We now consider the $N\to\infty$ limit of eq.~\eqref{eq:partition_function_relu_orthants}. Unlike in the linear case, we must take some care when applying the saddle point approximation here. In particular, we would like to first get the partition function into a form that is more amenable to analysis. Our first step will therefore be to construct a continuum approximation for $k_l$.

\begin{result}
As $N\to\infty$ the sum over orthants in eq.~\eqref{eq:partition_function_relu_orthants} can be converted to an integral and the $\Gamma$ functions can be approximated to give,
\begin{align}
Q &= \int[dk_l][dr_+^l][dr_-^l]\exp\Bigg[-\frac12\sum_{l=0}^L\Bigg(\frac{(r_+^l)^2 + (r^l_-)^2}{\sigma_w^2N^{-1}(r^{l-1}_+)^2 + \sigma_b^2} + k_l(3\log k_l - 2\log r_-^l) \nonumber\\
&\hspace{4pc}+ N\log(\sigma_w^2N^{-1}(r^l_+)^2+\sigma_b^2) + (N-k_l)(3\log(N-k_l) - 2\log r_+^l)\Bigg)\Bigg]\label{eq:partition_relu_continuum}
\end{align}
where $k_l$ is now a continuously varying field.
\end{result}
\begin{proof}
In the large $N$ limit we first note that the sum over orthants will concentrate about $k_l = N/2$. Moreover the product of binomial coefficients and $\Gamma$ functions can be approximated using Stirling's approximation. We therefore aim to approximate the sum by an integrand in the large $N$ limit. To do this first define $\Delta k = (2/\sqrt{\pi})^N$ noting that $\Delta k\to 0$ as $N\to\infty$. In the large $N$ limit the sum is identically a Riemann sum and so we may write (taking liberties to add/subtract 1 when convenient),
\begin{align}
Q &\approx 2\prod_l \int dk_l\frac{\Gamma(N+1)}{\Gamma(N-k_l+1)\Gamma(k_l+1)\Gamma\left(\frac{N-k_l}2\right)\Gamma\left(\frac{k_l}2\right)}\int dr^l_+ dr^l_-(r^l_+)^{N-k_l-1}(r^l_-)^{k_l-1}\nonumber\\
&\hspace{6pc}\times\exp\left[-\frac12\sum_{l=0}^L\left(\frac{(r_+^l)^2 + (r^l_-)^2}{\sigma_w^2N^{-1}(r^{l-1}_+)^2 + \sigma_b^2} + N\log(\sigma_w^2N^{-1}(r^l_+)^2+\sigma_b^2)\right)\right]\\
&\approx \prod_l \int dk_l\frac{1}{\Gamma(N-k_l+1)\Gamma(k_l+1)\Gamma\left(\frac{N-k_l}2+1\right)\Gamma\left(\frac{k_l}2+1\right)}\int dr^l_+ dr^l_-(r^l_+)^{N-k_l}(r^l_-)^{k_l}\nonumber\\&\hspace{6pc}\times\exp\left[-\frac12\sum_{l=0}^L\left(\frac{(r_+^l)^2 + (r^l_-)^2}{\sigma_w^2N^{-1}(r^{l-1}_+)^2 + \sigma_b^2}+ N\log(\sigma_w^2N^{-1}(r^l_+)^2+\sigma_b^2)\right)\right]\\
&\approx\int[dk_l][dr_+^l][dr_-^l]\left(\frac{2^{1/3}e}{N-k_l}\right)^{\frac32(N-k_l)}\left(\frac{2^{1/3}e}{k_l}\right)^{\frac32k_l}\exp\Bigg[-\frac12\sum_{l=0}^L\Bigg(\frac{(r_+^l)^2 + (r^l_-)^2}{\sigma_w^2N^{-1}(r^{l-1}_+)^2 + \sigma_b^2}\nonumber\\
&\hspace{3pc} + N\log(\sigma_w^2N^{-1}(r^l_+)^2+\sigma_b^2) - 2(N-k_l)\log r_+^l - 2k_l\log r_-^l\Bigg)\Bigg]\\
&= \int[dk_l][dr_+^l][dr_-^l]\exp\Bigg[-\frac12\sum_{l=0}^L\Bigg(\frac{(r_+^l)^2 + (r^l_-)^2}{\sigma_w^2N^{-1}(r^{l-1}_+)^2 + \sigma_b^2}+ N\log(\sigma_w^2N^{-1}(r^l_+)^2+\sigma_b^2) \nonumber\\&\hspace{6pc}+ (N-k_l)(3\log(N-k_l) - 2\log r_+^l) + k_l(3\log k_l - 2\log r_-^l)\Bigg)\Bigg].
\end{align}
Note that we have neglected the sub-leading constant factor in Stirling's approximation. Thus, we see that the anisotropy of the rectified linear unit causes us to have three independent fields that must be dealt with.
\end{proof}

We now compute the saddle point approximation to eq.~\eqref{eq:partition_relu_continuum}. As in the case of the stochastic linear network we begin by assuming the existence of a uniform solution with $r_+ = r^l_+$, $r_- = r^l_-$, and $k = k_l$. This leads us to our next result.

\begin{result}
For $\sigma_w^2 < 2$ rectified linear stochastic networks have a uniform configuration of fields that minimizes the energy. This minimum configuration has,
\begin{equation}
r_{+/-} = \sqrt{\frac{N\sigma_b^2}{2(1-\sigma_w^2/2)}} \hspace{5pc} k = \frac N2.
\end{equation}
\end{result}
\begin{proof}
In the case of the rectified linear network we notice that the energy function will be,
\begin{align}
\mathcal L = \frac12\sum_{l=0}^L\Bigg(&\frac{(r_+^l)^2 + (r^l_-)^2}{\sigma_w^2N^{-1}(r^{l-1}_+)^2 + \sigma_b^2}+ N\log(\sigma_w^2N^{-1}(r^l_+)^2+\sigma_b^2)\nonumber\\ &+ (N-k_l)(3\log(N-k_l) - 2\log r_+^l) + k_l(3\log k_l - 2\log r_-^l)\Bigg).
\end{align}
Given the anzats of a constant solution we seek a minimum satisfying the equations,
\begin{align}
&\frac{\partial H}{\partial r_+} = \frac{(1+\sigma_w^2)r_+}{\sigma_w^2N^{-1}r_+^2 + \sigma_b^2} - \frac{r_+^2 + r_-^2}{(\sigma_w^2 N^{-1}r_+^2 + \sigma_b^2)^2}\sigma_w^2N^{-1}r_+ - \frac{N-k_l}{r_+^l} = 0\label{eq:relu_saddle_1}\\
&\frac{\partial H}{\partial r_-} = \frac{r_-}{\sigma_w^2 N^{-1}r_+^2 + \sigma_b^2} - \frac k{r_-} = 0\label{eq:relu_saddle_2}\\
&\frac{\partial H}{\partial k}  = \log r_+ - \log r_- - \frac32\log(N-k) + \frac 32 \log k = 0\label{eq:relu_saddle_3}
\end{align}
These equations can be solved straightforwardly to give the required result. While the extremum of the saddle point approximation is qualitatively identical to the linear network we expect fluctuations in this case to be quite different. We can see that this will be the case first and foremost because we now have three fields instead of a single field. Fluctuations in these three directions will interact in interesting and measurable ways.
\end{proof}

Next we compute fluctuations about the saddle point solution. To do this, as before, we make the change of variables $k_l = k + \epsilon_k^l$, $r_+^l = r + \epsilon_+^l$, and $r_-^l = r+ \epsilon_-^l$. This leads us to our main result for rectified linear networks.

\begin{result} 
Small fluctuations of $r^l_{+/-}$ and $k^l$ about the saddle point solution are described by the energy in excess of the energy of the constant solution 
\begin{equation}
U = \mathcal L(\{r_+ + \epsilon^l_+\},\{r_- + \epsilon^l_-\},\{k + \epsilon^l_k\}) - \mathcal L(r_+,r_-,k),
\end{equation}
which may be expanded to quadratic order to give,
\begin{equation}
U =  -\frac12\sum_{l=0}^L\Bigg((1+\sigma_w^4/2)(\tilde\epsilon_+^l)^2 + (\tilde\epsilon_-^l)^2 + 3(\tilde\epsilon_k^l)^2  + \tilde\epsilon_k^l(\tilde\epsilon_+^l-\tilde\epsilon_-^l) - \sigma_w^2\tilde\epsilon_+^{l-1}(\tilde\epsilon_+^l + \tilde\epsilon_-^l)\Bigg)\label{eq:stochastic_network_relu_real_space}
\end{equation} 
where $\tilde\epsilon^l_{+/-} = \sqrt{2(1-\sigma_w^2/2)/\sigma_b^2}\epsilon^l_{+/-}$ and $\tilde\epsilon_k^l = \epsilon^l_k/\sqrt N$.
\end{result}
\begin{proof}
Before we begin we define $r = r_{+/-}$,
\begin{equation}
\sigma_w^2N^{-1}r^2 + \sigma_b^2 = \frac12\frac{\sigma_w^2\sigma_b^2 + 2(1-\sigma_w^2/2)\sigma_b^2}{1-\sigma_w^2/2} = \frac{\sigma_b^2}{1-\sigma_w^2/2} = \eta
\end{equation}
and $\alpha = \sigma_w^2N^{-1}$. Expanding the energy directly we find that,
\begin{align}
\mathcal L &= \frac12\sum_{l=0}^L\Bigg(\frac{2r^2 + 2r(\epsilon_+^l + \epsilon_-^l) + (\epsilon_+^l)^2 + (\epsilon_-^l)^2}{\eta +  \sigma_w^2N^{-1}\epsilon_+^{l-1}(2r+\epsilon_+^{l-1})} + N\log(\eta + \sigma_w^2N^{-1}\epsilon_+^l(2r+\epsilon_+^l))\nonumber\\
&\hspace{10pc}+ (N/2 - \epsilon_k^l)(3\log(N/2-\epsilon_k^l) - 2\log(r + \epsilon_+^l)) \nonumber\\
&\hspace{10pc}+ (N/2+\epsilon_k^l)(3\log(k+\epsilon_k^l) - 2\log(r + \epsilon_-^l) \Bigg)\Bigg]\\
&\approx \frac12\sum_{l=0}^L\Bigg(\frac N{r^2}(1+\alpha^2N^2/2)(\epsilon_+^l)^2 + \frac N{r^2}(\epsilon_-^l)^2 + \frac 6N(\epsilon_k^l)^2\nonumber\\
&\hspace{6pc} + \frac2r\epsilon_k^l(\epsilon_+^l - \epsilon_-^l) - \frac{\alpha N^2}{r^2}\epsilon_+^{l-1}(\epsilon_+^l + \epsilon_-^l)\Bigg)\Bigg].
\end{align}
Substituting back in for $\alpha$ and $r$ we arrive at the equation,
\begin{align}
&= \frac12\sum_{l=0}^L\Bigg(\frac{2(1-\sigma_w^2/2)(1+\sigma_w^4/2)}{\sigma_b^2}(\epsilon_+^l)^2 + \frac{2(1-\sigma_w^2/2)}{\sigma_b^2}(\epsilon_-^l)^2 + \frac 6N(\epsilon_k^l)^2  \nonumber\\
&\hspace{4pc}+ 2\sqrt{\frac{2(1-\sigma_w^2/2)}{N\sigma_b^2}}\epsilon_k^l(\epsilon_+^l-\epsilon_-^l) - \frac{2\sigma_w^2(1-\sigma_w^2/2)}{\sigma_b^2}\epsilon_+^{l-1}(\epsilon_+^l + \epsilon_-^l)\Bigg)\Bigg]. 
\end{align}
We will make the change of variables $\epsilon^l_k \to \epsilon^l_k/\sqrt{N}$, $\epsilon^l_{+/-} \to \sqrt{2(1-\sigma_w^2/2)/\sigma_b^2}\epsilon^l_{+/-}$ in which case we may rewrite the above equation in normal coordinates as,
\begin{align}
Q &= \int[d\epsilon_k^l][d\epsilon_+^l][d\epsilon_-^l]\exp\Bigg[-\frac12\sum_{l=0}^L\Bigg((1+\sigma_w^4/2)(\epsilon_+^l)^2 + (\epsilon_-^l)^2 + 3(\epsilon_k^l)^2 \nonumber\\
&\hspace{9pc} + \epsilon_k^l(\epsilon_+^l-\epsilon_-^l) - \sigma_w^2\epsilon_+^{l-1}(\epsilon_+^l + \epsilon_-^l)\Bigg)\Bigg]
\end{align}
which completes the proof.
\end{proof}

Next we change variables to Fourier basis and - as in the case of the linear network - we find that the Fourier modes are decoupled. Unlike in the case of the linear network here the different fields remain coupled yielding. Small fluctuations in the Fourier basis are therefore described by a Gaussian distribution with nontrivial covariance matrix. In the main text we checked this covariance matrix against measurements from numerical experiments.

\begin{result}
Changing variables to Fourier basis in eq.~\eqref{eq:stochastic_network_relu_real_space} yields the distribution over Fourier modes defined by the energy,
\begin{equation}
U = -\frac12\sum_q 
\begin{pmatrix}\epsilon_+^{-q} & \epsilon_-^{-q} & \epsilon_k^{-q}\end{pmatrix}\begin{pmatrix}
1+\sigma_w^4/2 -\sigma_w^2\cos q & -\frac12\sigma_w^2e^{-iq} & \frac12 \\
-\frac12\sigma_w^2e^{iq} & 1 & -\frac12 \\
\frac12 & -\frac12 & 3
\end{pmatrix}\begin{pmatrix}\epsilon_+^{q} \\ \epsilon_-^{q} \\ \epsilon_k^{q}\end{pmatrix}
\end{equation}
which is well defined as a Gaussian with inverse covariance matrix,
\begin{equation}
\Sigma^{-1}(q) = \begin{pmatrix}
1+\sigma_w^4/2 -\sigma_w^2\cos q & -\frac12\sigma_w^2e^{-iq} & \frac12 \\
-\frac12\sigma_w^2e^{iq} & 1 & -\frac12 \\
\frac12 & -\frac12 & 3
\end{pmatrix}.
\end{equation}
\end{result}
\begin{proof}
We now change variables to the Fourier basis and write 
\begin{equation}
\epsilon^l_+ = \sum_l e^{-iql}\epsilon^q_+ \hspace{3pc} \epsilon^l_- = \sum_l e^{-iql}\epsilon^q_- \hspace{3pc} \epsilon^l_k = \sum_l e^{-iql}\epsilon^q_k
\end{equation}
where each of the $q_x$ are summed between $0$ and $2\pi$ in steps of $2\pi/L$. With this change of variables eq.~\eqref{eq:stochastic_network_relu_real_space} can be rewritten as,
\begin{align}
Q &= \int[d\epsilon_k^q][d\epsilon_+^q][d\epsilon_-^q]\exp\Bigg[-\frac12\sum_{l=0}^L\sum_{q,q'}\Bigg((1+\sigma_w^4/2)\epsilon_+^q\epsilon_+^{q'} + \epsilon_-^q\epsilon_-^{q'} + 3\epsilon_k^q\epsilon_k^{q'}  \nonumber\\
&\hspace{6pc} + \epsilon_k^q(\epsilon_+^{q'}-\epsilon_-^{q'}) - \sigma_w^2\epsilon_+^q(\epsilon_+^{q'} + \epsilon_-^{q'})e^{-iq}\Bigg)e^{i(q+q')l}\Bigg] \\
&= \int[d\epsilon_k^q][d\epsilon_+^q][d\epsilon_-^q]\exp\Bigg[-\frac12\sum_q\Bigg((1+\sigma_w^4/2 - \sigma_w^2\cos q)|\epsilon_+^q|^2 + |\epsilon_-^{q}|^2 + 3|\epsilon_k^q|^2  \nonumber\\
&\hspace{6pc} + \frac12\left(\epsilon_k^q(\epsilon_+^{q})^*-\epsilon_k^q(\epsilon_-^{q})^* - \sigma_w^2\epsilon_+^q(\epsilon_-^q)^*e^{-iq} + \text{h.c.}\right)\Bigg)\Bigg]
\end{align}
where $\text{h.c.}$ refers to the Hermitian conjugate. This may be rewritten in matrix form as,
\begin{equation}
Q = \int[d\epsilon_k^q][d\epsilon_+^q][d\epsilon_-^q]\exp\left[-\frac12\sum_q 
\begin{pmatrix}\epsilon_+^{-q} & \epsilon_-^{-q} & \epsilon_k^{-q}\end{pmatrix}\Sigma^{-1}(q)\begin{pmatrix}\epsilon_+^{q} \\ \epsilon_-^{q} \\ \epsilon_k^{q}\end{pmatrix}\right]
\end{equation}
which is well defined as a Gaussian with inverse covariance matrix,
\begin{equation}
\Sigma^{-1}(q) = \begin{pmatrix}
1+\sigma_w^4/2 -\sigma_w^2\cos q & -\frac12\sigma_w^2e^{-iq} & \frac12 \\
-\frac12\sigma_w^2e^{iq} & 1 & -\frac12 \\
\frac12 & -\frac12 & 3
\end{pmatrix}.
\end{equation}
As before this allows us to test the predictions of the theory.
\end{proof}

Finally, we construct the effective field theory under the assumption that all of the fields are slowly varying with respect to a single layer.
\begin{result}
The effective field theory for long wavelength fluctuations of the rectified linear network will be given by,
\begin{align}
U = \frac12\int dx\Bigg[&(1 - \sigma_w^2 +\sigma_w^4/2)(\epsilon_+(x))^2 + (\epsilon_-(x))^2 + 3(\epsilon_k(x))^2 + \epsilon_k(x)(\epsilon_+(x) - \epsilon_-(x)) \nonumber\\
& + \sigma_w^2\epsilon_+(x)\epsilon_-(x) + \sigma_w^2\left(\frac{\partial\epsilon_+(x)}{\partial x}\right)^2 + \sigma_w^2\frac{\partial \epsilon_+(x)}{\partial x}\epsilon_-(x)  \Bigg].
\end{align}
\end{result}
\begin{proof}
We begin by noting that eq.~\eqref{eq:stochastic_network_relu_real_space} can be rewritten as,
\begin{align}
U &=  -\frac12\sum_{l=0}^L\Bigg((1+\sigma_w^4/2)(\tilde\epsilon_+^l)^2 + (\tilde\epsilon_-^l)^2 + 3(\tilde\epsilon_k^l)^2  + \tilde\epsilon_k^l(\tilde\epsilon_+^l-\tilde\epsilon_-^l) - \sigma_w^2\tilde\epsilon_+^{l-1}(\tilde\epsilon_+^l + \tilde\epsilon_-^l)\Bigg)\\
&= -\frac12\sum_{l=0}^L\Bigg((1 - \sigma_w^2 +\sigma_w^4/2)(\tilde\epsilon_+^l)^2 + (\tilde\epsilon_-^l)^2 + 3(\tilde\epsilon_k^l)^2 + \sigma_w^2\tilde \epsilon_+^l\tilde \epsilon_-^l  + \tilde\epsilon_k^l(\tilde\epsilon_+^l-\tilde\epsilon_-^l) \nonumber\\&\hspace{13.5pc} + \frac12\sigma_w^2(\tilde\epsilon_+^l - \tilde\epsilon_+^{l-1})^2 + \sigma_w^2(\tilde\epsilon_+^l - \tilde\epsilon_+^{l-1})\tilde\epsilon_-^l\Bigg).
\end{align}
As before we when fluctuations are slowly varying on the order of a single layer of the network we can interpret $\tilde\epsilon^l_+ - \tilde\epsilon^{l-1}_+ \approx \partial \epsilon_+/\partial x$ and we can approximate the sum by an integral. Together these approximations give,
 \begin{align}
U = \frac12\int dx\Bigg[&(1 - \sigma_w^2 +\sigma_w^4/2)(\epsilon_+(x))^2 + (\epsilon_-(x))^2 + 3(\epsilon_k(x))^2 + \epsilon_k(x)(\epsilon_+(x) - \epsilon_-(x)) \nonumber\\
& + \sigma_w^2\epsilon_+(x)\epsilon_-(x) + \sigma_w^2\left(\frac{\partial\epsilon_+(x)}{\partial x}\right)^2 + \sigma_w^2\frac{\partial \epsilon_+(x)}{\partial x}\epsilon_-(x)  \Bigg].
\end{align}
as expected.
\end{proof}

\bibliography{example_paper}

\begin{thebibliography}{21}
\providecommand{\natexlab}[1]{#1}
\providecommand{\url}[1]{\texttt{#1}}
\expandafter\ifx\csname urlstyle\endcsname\relax
  \providecommand{\doi}[1]{doi: #1}\else
  \providecommand{\doi}{doi: \begingroup \urlstyle{rm}\Url}\fi

\bibitem[Buice and Chow(2013)]{buice2013}
Michael~A Buice and Carson~C Chow.
\newblock Beyond mean field theory: statistical field theory for neural
  networks.
\newblock \emph{Journal of Statistical Mechanics: Theory and Experiment},
  2013\penalty0 (03):\penalty0 P03003, 2013.
\newblock URL \url{http://stacks.iop.org/1742-5468/2013/i=03/a=P03003}.

\bibitem[Chaikin and Lubensky(2000)]{chaikin2000}
Paul~M Chaikin and Tom~C Lubensky.
\newblock \emph{Principles of condensed matter physics}.
\newblock Cambridge university press, 2000.

\bibitem[Cho and Saul(2009)]{cho2009}
Youngmin Cho and Lawrence~K Saul.
\newblock Kernel methods for deep learning.
\newblock In \emph{Advances in neural information processing systems}, pages
  342--350, 2009.

\bibitem[Choromanska et~al.(2015)Choromanska, Henaff, Mathieu, Arous, and
  LeCun]{choromanska2015}
Anna Choromanska, Mikael Henaff, Michael Mathieu, G{\'e}rard~Ben Arous, and
  Yann LeCun.
\newblock The loss surfaces of multilayer networks.
\newblock In \emph{AISTATS}, 2015.

\bibitem[{Daniely} et~al.(2016){Daniely}, {Frostig}, and {Singer}]{daniely2016}
A.~{Daniely}, R.~{Frostig}, and Y.~{Singer}.
\newblock {Toward Deeper Understanding of Neural Networks: The Power of
  Initialization and a Dual View on Expressivity}.
\newblock \emph{arXiv:1602.05897}, 2016.

\bibitem[Daniely et~al.(2017)Daniely, Frostig, Gupta, and Singer]{daniely2017}
Amit Daniely, Roy Frostig, Vineet Gupta, and Yoram Singer.
\newblock Random features for compositional kernels.
\newblock \emph{CoRR}, abs/1703.07872, 2017.
\newblock URL \url{http://arxiv.org/abs/1703.07872}.

\bibitem[Ermentrout and Cowan(1979)]{ermentrout1979}
G~Bard Ermentrout and Jack~D Cowan.
\newblock A mathematical theory of visual hallucination patterns.
\newblock \emph{Biological cybernetics}, 34\penalty0 (3):\penalty0 137--150,
  1979.

\bibitem[Hinton et~al.(2012)Hinton, Deng, Yu, Dahl, Mohamed, Jaitly, Senior,
  Vanhoucke, Nguyen, Sainath, et~al.]{hinton2012}
Geoffrey Hinton, Li~Deng, Dong Yu, George~E. Dahl, Abdel-rahman Mohamed,
  Navdeep Jaitly, Andrew Senior, Vincent Vanhoucke, Patrick Nguyen, Tara~N
  Sainath, et~al.
\newblock Deep neural networks for acoustic modeling in speech recognition: The
  shared views of four research groups.
\newblock \emph{IEEE Signal Processing Magazine}, 29\penalty0 (6):\penalty0
  82--97, 2012.

\bibitem[Jaynes(1957)]{jaynes1957}
Edwin~T Jaynes.
\newblock Information theory and statistical mechanics.
\newblock \emph{Physical Review}, 106\penalty0 (4):\penalty0 620, 1957.

\bibitem[Krizhevsky et~al.(2012)Krizhevsky, Sutskever, and
  Hinton]{krizhevsky2012}
Alex Krizhevsky, Ilya Sutskever, and Geoffrey~E Hinton.
\newblock Imagenet classification with deep convolutional neural networks.
\newblock In F.~Pereira, C.~J.~C. Burges, L.~Bottou, and K.~Q. Weinberger,
  editors, \emph{Advances in Neural Information Processing Systems 25}, pages
  1097--1105. Curran Associates, Inc., 2012.

\bibitem[M{\'e}zard et~al.(1987)M{\'e}zard, Parisi, and Virasoro]{mezard1987}
Marc M{\'e}zard, Giorgio Parisi, and Miguel Virasoro.
\newblock \emph{Spin glass theory and beyond: An Introduction to the Replica
  Method and Its Applications}, volume~9.
\newblock World Scientific Publishing Co Inc, 1987.

\bibitem[Neal(1996)]{neal1996priors}
Radford~M Neal.
\newblock Priors for infinite networks.
\newblock In \emph{Bayesian Learning for Neural Networks}, pages 29--53.
  Springer, 1996.

\bibitem[Neal(2012)]{neal2012}
Radford~M Neal.
\newblock \emph{Bayesian learning for neural networks}, volume 118.
\newblock Springer Science \& Business Media, 2012.

\bibitem[{Poole} et~al.(2016){Poole}, {Lahiri}, {Raghu}, {Sohl-Dickstein}, and
  {Ganguli}]{poole2016}
B.~{Poole}, S.~{Lahiri}, M.~{Raghu}, J.~{Sohl-Dickstein}, and S.~{Ganguli}.
\newblock {Exponential expressivity in deep neural networks through transient
  chaos}.
\newblock \emph{Neural Information Processing Systems}, 2016.

\bibitem[{Raghu} et~al.(2017){Raghu}, {Poole}, {Kleinberg}, {Ganguli}, and
  {Sohl-Dickstein}]{raghu2016}
M.~{Raghu}, B.~{Poole}, J.~{Kleinberg}, S.~{Ganguli}, and J.~{Sohl-Dickstein}.
\newblock {On the expressive power of deep neural networks}.
\newblock \emph{International Conference on Machine Learning}, 2017.

\bibitem[{Saxe} et~al.(2014){Saxe}, {McClelland}, and {Ganguli}]{saxe2014}
A.~M. {Saxe}, J.~L. {McClelland}, and S.~{Ganguli}.
\newblock {Exact solutions to the nonlinear dynamics of learning in deep linear
  neural networks}.
\newblock \emph{International Conference on Learning Representations}, 2014.

\bibitem[Schneidman et~al.(2006)Schneidman, Berry, Segev, and Bialek]{elad2006}
Elad Schneidman, Michael~J. Berry, Ronen Segev, and William Bialek.
\newblock Weak pairwise correlations imply strongly correlated network states
  in a neural population.
\newblock \emph{Nature}, 440\penalty0 (7087):\penalty0 1007--1012, 04 2006.
\newblock URL \url{http://dx.doi.org/10.1038/nature04701}.

\bibitem[Schoenholz et~al.(2017)Schoenholz, Gilmer, Ganguli, and
  Sohl-Dickstein]{schoenholz2016}
Samuel~S Schoenholz, Justin Gilmer, Surya Ganguli, and Jascha Sohl-Dickstein.
\newblock {Deep Information Propagation}.
\newblock \emph{International Conference on Learning Representations}, 2017.

\bibitem[Shazeer et~al.(2017)Shazeer, Mirhoseini, Maziarz, Davis, Le, Hinton,
  and Dean]{shazeer2017}
Noam Shazeer, Azalia Mirhoseini, Krzysztof Maziarz, Andy Davis, Quoc Le,
  Geoffrey Hinton, and Jeff Dean.
\newblock Outrageously large neural networks: The sparsely-gated
  mixture-of-experts layer.
\newblock 2017.
\newblock URL \url{https://openreview.net/pdf?id=B1ckMDqlg}.

\bibitem[Weinberg(1967)]{weinberg1967}
Steven Weinberg.
\newblock A model of leptons.
\newblock \emph{Phys. Rev. Lett.}, 19:\penalty0 1264--1266, Nov 1967.
\newblock \doi{10.1103/PhysRevLett.19.1264}.
\newblock URL \url{http://link.aps.org/doi/10.1103/PhysRevLett.19.1264}.

\bibitem[Wu et~al.(2016)Wu, Schuster, Chen, Le, Norouzi, Macherey, Krikun, Cao,
  Gao, Macherey, et~al.]{wu2016}
Yonghui Wu, Mike Schuster, Zhifeng Chen, Quoc~V. Le, Mohammad Norouzi, Wolfgang
  Macherey, Maxim Krikun, Yuan Cao, Qin Gao, Klaus Macherey, et~al.
\newblock Google's neural machine translation system: Bridging the gap between
  human and machine translation.
\newblock \emph{arXiv preprint arXiv:1609.08144}, 2016.

\end{thebibliography}
\bibliographystyle{plainnat}

\end{document}